%% file: main.tex
\documentclass[11pt]{article}

\usepackage[suppress]{color-edits}
\addauthor{pco}{blue}
\addauthor{rdk}{red}

\usepackage{natbib}

\usepackage{amsmath}
\usepackage{amsfonts}
\usepackage{amssymb}
\usepackage{amsthm}

\usepackage[dvipsnames]{xcolor}
\usepackage[
pagebackref,
colorlinks=true,
urlcolor=blue,
linkcolor=blue,
citecolor=OliveGreen,
]{hyperref}
\usepackage{url}            % simple URL typesetting
\usepackage{booktabs}       % professional-quality tables
\usepackage{amsfonts}       % blackboard math symbols
\usepackage{nicefrac}       % compact symbols for 1/2, etc.
\usepackage{microtype}      % microtypography
\usepackage{bbm}
\usepackage{pgfplots}
\usepackage{thmtools,thm-restate}
\usepackage{xspace}
\usepackage{authblk}
\usepackage{caption}
\usepackage{subcaption}
\usepackage{tikz}
\usepackage{paralist}
\usepackage{libertine}
\usepackage[nameinlink]{cleveref}
\usepackage{enumitem}

\input{notations.tex}
\usepackage{mkolar_definitions}

\usepackage[margin=1in]{geometry}
\usepackage{txfonts}
\usepackage[T1]{fontenc}

\title{Faster Recalibration of an Online Predictor via Approachability}
\author[1]{Princewill Okoroafor}
\author[1]{Robert Kleinberg}
\author[1]{Wen Sun}
\affil[1]{Cornell University, Ithaca, NY\protect\\
{\footnotesize \texttt{pco9@cornell.edu, rdk@cs.cornell.edu, ws455@cornell.edu}}}
\date{}

\begin{document}

\maketitle

\begin{abstract}
  Predictive models in ML need to be trustworthy and reliable, which often at the very least means outputting calibrated probabilities. This can be particularly difficult to guarantee in the online prediction setting when the outcome sequence can be generated adversarially. In this paper we introduce a technique using Blackwell’s approachability theorem for taking an online predictive model which might not be calibrated and transforming its predictions to calibrated predictions without much increase to the loss of the original model. Our proposed algorithm achieves calibration and accuracy at a faster rate than existing techniques \citep{kuleshov} and is the first algorithm to offer a flexible tradeoff between calibration error and accuracy in the online setting. We demonstrate this by characterizing the space of jointly achievable calibration and regret using our technique. 
\end{abstract}

\input{sections/intro}
\input{sections/background}
\input{sections/approachability}
\input{sections/rates}

% If you use BibTeX in apalike style, activate the following line:
\bibliographystyle{apalike}
\bibliography{ref}

\input{sections/supplementary}

\input{sections/minimax-appendix}

\end{document}

%% file: notations.tex
\usepackage{float}
\usepackage[algo2e,ruled,linesnumbered]{algorithm2e}
\usepackage{graphicx,algorithm,algorithmic}

\newcommand{\I}[1]{\mathbb{I}\left[#1\right]}       % Indicator
\newcommand{\A}{\mathcal{A}}    % Algorithm

\newcommand{\reals}{{\mathbb{R}}}

\def\epsilon{\varepsilon}
\def\zero{\mathbf{0}}

\def\dist{\textnormal{\texttt{dist}}}
\def\reals{\mathbb{R}}
\def\X{\mathcal{X}}
\def\Y{\mathcal{Y}}
\def\A{\mathcal{A}}

\def\lb{\boldsymbol{\ell}}

\def\wv{\mathbf{w}}

\def\K{\mathcal{K}}
\def\I{\mathbb{I}}

\def\thb{\boldsymbol{\theta}}

\def\E{\mathbb{E}}

\def\approachset{\mathcal{S}^m_{\text{approach}}}

%% file: sections/intro.tex
\section{Introduction}\label{sec:intro}

In the online learning setting, a predictive model, also known as a forecaster, gives a probability value prediction at each time step, and its performance is evaluated based on a loss function. For the class of loss function known as a proper scoring rule, the only way to minimize that score is to predict the true probabilities of an outcome. 
For most prediction problems we do not know how to compute the true probabilities of outcomes, and the best we can do is to use a trained model (e.g., a deep neural network or contextual bandit algorithm) attaining a low scoring-rule loss without necessarily minimizing it.
However, most 
% RDK edit: was "loss functions"
training methods for predictive models do not guarantee calibrated probability values. %(Good place for an example). 
There has been a large body of work highlighting the need for calibrated probability estimates (i.e., models that are able to assess their uncertainty) \citep{10.1136/amiajnl-2011-000291, degroot1983comparison} and on how to obtain these calibrated probability estimates \citet{foster_proof_1999}. In the offline setting, this is generally done by some post-processing of the data to remap the probability values to calibrated probability estimates in a way that minimizes the increase in loss, such as by post-hoc calibration or recalibration. 
%\pcocomment{consider adding citation here} 
In contrast, in the online prediction setting, little work has been done on this subject. Recently, \citet{kuleshov} and \citet{calibeat} have presented various approaches for taking an online predictive model and transforming its predictions without major increase in loss. \citet{kuleshov} introduced this problem as an online recalibration problem, and provided an algorithm for achieving epsilon accuracy relative to the loss function using a connection between calibration and internal regret. 
% RDK edit: rephrasing the following sentences...
% However, their result can be significantly improved by using Blackwell’s Approachability Theorem. In this paper, we present an approachability algorithm that achieves recalibration at a much faster rate than the internal regret minimization algorithm by \citet{kuleshov}. 
In this paper, we show that their result can be significantly improved by using Blackwell’s Approachability Theorem. We present an algorithm, making use of approachability, that achieves recalibration at a much faster rate than the internal regret minimization algorithm by
\citet{kuleshov}. 
We also characterize the achievable amount of calibration and regret as a function of the time horizon; more precisely, 
% RDK added
we study for which exponents $a,b$ does there exist a forecasting algorithm that guarantees at most $T^a$ calibration error and no more than $T^b$ regret relative to scoring rule loss functions. We provide the first algorithm that offers a flexible tradeoff between calibration error and regret in the online setting. \pcodelete{We generalize our result to $L_p$ calibration
% RDK added
for $p \ge 1$. }
% and multiclass calibration We also show that our technique can be used to achieve a stronger notion of recalibration known as calibeating, introduced by \citet{calibeat}.

\subsection{Motivation}\label{sec:motivation}

\paragraph{Calibrating probability predictions} 
As the prevalence of machine learning systems in decision-making settings grows, it is essential that the predictions they provide are trustworthy, especially in applications where the confidence associated with the prediction is at least as important as the prediction itself. Neural networks have been found to be poor at assessing their own uncertainty \citep{pmlr-v70-guo17a}, and as a result, may output probability values that do not match the true probabilities of outcomes. This can have serious consequences; machine learning systems have been known to propagate unintended but harmful discrimination, as shown by \citet{pmlr-v81-buolamwini18a} for image classification and \citet{Bolukbasi} for natural language tasks. 
One proposed method for addressing the issue of assessing uncertainty is calibration \citep{pmlr-v80-hebert-johnson18a}. Calibration requires that the probability estimates from the ML model match the true distribution of the outcome; for example, for a binary class, if a model outputs a probability of 0.3 a certain number of times, the proportion of true outcomes should be 30 percent across the total instances when the model predicted 0.3. In the online setting, many works have proposed techniques for how to achieve calibrated probability estimates, even in the adversarial setting \citep{foster_proof_1999, mannor, pmlr-v19-abernethy11b}.

\paragraph{Limitations of calibration} While calibration is a useful property for online predictors to have, calibration is not sufficient and does not fully reflect domain specific knowledge. For example, consider two ML weather forecasters. Suppose the true outcome is that it rains 
% RDK edit
once every two days. Forecaster 1 predicts 50 percent chance of rain every day, and Forecaster 2 predicts 0 percent chance of rain on the days it does not rain and 100 percent on the days it does. Observe that both of these forecasters are equally calibrated; however, the second forecaster is a better predictor of the likelihood of rain. Calibration does not capture this fact. Although calibration does not imply accuracy, accuracy does imply calibration, simply because being accurate requires an understanding of the outcome distribution.
This is why, in practice, proper scoring rules are used to assess the accuracy of predictions \citep{gneiting}. 

\paragraph{Incorporating expert/domain-specific knowledge in online prediction models} Forecaster 2 is an example of a forecaster that reflects domain-specific knowledge and is also calibrated. However, it is also possible for a forecaster that acts on domain specific knowledge to be poorly calibrated. 
Consider a third forecaster in the same weather prediction setting which predicts 20 percent chance of rain on the days it does not rain, and 80 percent chance of rain on the days that it does. This predictor is poorly calibrated, because it incurs a calibration error of 0.2 for every decision. However, compared to Forecaster 1, its predictions still reflect a domain-specific understanding of the probability distribution. The goal of our work is to take a model such as this third forecaster and transform its predictions in an online setting to achieve calibration while still making decisions that are informed by domain knowledge.

% RDK edit: \subsection{Statement of Technical Problem}
\subsection{Problem formulation}

In this paper, we focus on a class of loss functions known as strictly proper scoring rules. We refer the reader to Section~\ref{sec:scoring} for an introduction on the subject.
%Suppose $S$ is a strictly proper scoring rule and

Consider an online prediction environment where the timing
of each round of the prediction process is as follows.
\begin{enumerate}
  \item An oracle reveals a prediction $q_t$.
  \item The algorithm must make a prediction $p_t$.
  \item The actual label $y_t \in \{0,1\}$ is revealed.
  \item The algorithm receives a score $S(p_t, y_t)$. 
\end{enumerate}
At the end of $T$ rounds, the following quantities are
calculated.
\begin{itemize}
  \item The forecaster's cumulative score is $S_f = \sum_{t=1}^T S(p_t,y_t)$.
  \item The oracle's cumulative score is $S_o = \sum_{t=1}^T S(q_t,y_t)$.
  \item The forecaster's 
  % RDK edit: regret is $S_f - S_o.$
  average regret is $\frac1T (S_f - S_o).$
  \item The forecaster's $\ell_1$-calibration error is
  % RDK edit: inserted 1/T factor...
  \[  \sum_{p \in [0,1]} \left|
    \tfrac1T \sum_{t=1}^T (y_t - p) \cdot \mathbbm{1}_{p_t=p} \right| .
  \]
  (Although written as a sum over all $p \in [0,1]$, the sum is
  actually finite because there are only finitely
  many $p$ for which the summand is nonzero.)
\end{itemize}
% RDK added:
For the sake of generality, our model makes no assumptions
about how the oracle's predictions are generated, except
that if the algorithm is randomized the oracle cannot 
anticipate the algorithm's \emph{future} coin-tosses.
This means, for example, that our simple prediction model 
subsumes more elaborate models in which the predictions 
$q_t$ are generated by a contextual bandit algorithm,
or by a pre-trained model such as a deep neural network,
using domain-specific features observed at time $t$ or
earlier.

% RDK edit:
Our work addresses the question: for which exponent
pairs $(a,b)$ is there a forecasting algorithm that guarantees
regret $\tilde{O}(T^a)$ and calibration error $\tilde{O}(T^b)$?
The purpose of this paper is to propose a method of tackling this question using Blackwell's Approachability Theorem.

% RDK edit: \subsection{Statement of our results}
\subsection{Our results}
For the notion of regret described above, we show that there is a family of approachability-based algorithms, parameterized by $\epsilon > 0$, that simultaneously achieves calibration $O(\epsilon + 1/\sqrt{\epsilon T} )$ and average regret $O(\epsilon^2 + 1/\sqrt{\epsilon T})$. This is a significant improvement from the result by \citet{kuleshov}, which achieves calibration $O(\epsilon  + 1/\sqrt{\epsilon^2 T})$ and average regret $O(\epsilon  + 1/\sqrt{\epsilon^2 T})$. The improved dependence on $\epsilon$ is significant in practice because it impacts how many samples, $T$,
are required in order to make the average regret less than some specified upper bound, $\delta$. For example, to make 
$\epsilon^2 + 1/\sqrt{\epsilon T}$ less than $\delta$
one would set $T = O(\delta^{-5/2})$ and 
$\epsilon = O(\delta^{1/2})$, whereas to 
make $\epsilon + 1/\sqrt{\epsilon^2 T}$ less 
than $\delta$ requires $T = O(\delta^{-4})$
and $\epsilon = O(\delta)$. 
% RDK edit:
% epsilon is often a function of the time horizon (often $\epsilon = T^{-\frac{1}{3}}$).
% This is equivalent to the difference between $T^{\frac{1}{3}}$ and $T^{\frac{2}{3}}$ regret, which is sufficiently notable to the online learning community.
For $\delta=0.1$ this amounts to the
difference between a few hundred samples
versus more than ten thousand.

By choosing $\epsilon$ 
% RDK deleted: and $\delta$ 
appropriately, we show that our algorithm can be designed to achieve the best known calibration upper bound of $T^{-\frac{1}{3}}$ while limiting regret to no more than $T^{-\frac{1}{3}}$.
If one is more interested in minimizing regret, we also show that 
% RDK edit: epsilon and delta 
$\epsilon$ can be chosen to achieve regret of $T^{-\frac{2}{5}}$ while limiting calibration error to no more than $T^{-\frac{1}{5}}$. 
The algorithm allows for a linear interpolation between these two bounds. That is, 
% RDK edit:
% we can achieve calibration error $T^{0.67}$ while keeping regret at no more than $T^{0.67}$, and regret $T^{0.6}$ while keeping calibration error at no more than $T^{0.8}$. 
for any $x$ in the interval $\left[ \frac13, \frac25 \right]$,
we can set $\epsilon=T^{-2x}$ to achieve calibration
$O(T^{2x-1})$ while simultaneously
achieving regret $O(T^{-x}).$
% and multiclass calibration, and show how this result can be used to achieve a stronger notion of calibeating \citep{calibeat}.  
%{\bf (TODO: what are the bounds for $p>1$?)}

\subsection{Comparison to prior work}
\paragraph{Calibration and proper scoring rules}
\citet{foster_proof_1999} first reduced calibration to approachability. Since then, a number of alternative proofs of calibration using reductions to approachability have emerged \citep{mannor,pmlr-v19-abernethy11b}. 
\rdkedit{Our work draws ideas and techniques from these papers, and extends those ideas with innovations specific to the task of online recalibration.}
Unlike in the standard setting of calibrated binary sequence prediction, the recalibration problem 
% our setting (``recalibration'') 
incorporates side information in the form of an oracle who makes a prediction at each timestep. In the standard calibration problem, the goal is to minimize calibration error. In the recalibration problem, the algorithm must attain two goals simultaneously: sublinear calibration error and sublinear regret relative to the oracle's predictions.
To achieve both of these objectives we need to modify the vector payoffs and the approachable set used in the standard reduction from calibration to approachability. The main technical \rdkedit{innovation}
% challenge we had to overcome 
in this work lies in verifying that the modified set is indeed approachable in the modified vector-payoff game. After showing that the modified set is indeed approachable, we rely on a reduction from approachability to Online Linear Optimization by \cite{pmlr-v19-abernethy11b} to construct an algorithm for recalibration. The geometry of our approachable set leads to quantitative bounds on calibration error and regret that improve upon the state of the art.

\paragraph{Recalibration in offline setting}
In the offline setting, calibrated predictions are usually constructed using methods such as Platt Scaling \citep{platt} and isotonic regression \citep{niculescu}. 
In the context of binary classification, these methods reduce the problem of outputting calibrated predictions to a one-dimensional regression problem. 
Given data $\{(x_i, y_i) \}_{i=1}^n$, they train a model $f(s)$ to predict $p_i = f(s)$ from uncalibrated scores $s_i = g(x_i)$ produced by a classifier $g$. These techniques are particularly suited for the offline setting where the training and the calibration phases of the algorithm can be separated and thus, do not apply in the online setting and can fail when the test distribution does not match the training distribution. Our results, on the other hand, are robust to adversarial manipulations.

\paragraph{Recalibration in online setting} \label{pg:recalibration}
\citet{kuleshov} present an algorithm for recalibration, that is, for achieving $\epsilon$ calibration and $\epsilon$ regret simultaneously at a rate of $1/\epsilon \sqrt{T}$. They achieve this by running $1/\epsilon$ many calibration algorithms in parallel for each prediction interval that the expert 
% RDK edit: (i.e., blackbox predictor) 
(called ``oracle'' in our work, ``blackbox predictor'' in theirs)
makes. This method works because calibrated predictors have been shown to minimize internal regret \citep{cesa2006prediction}. They are able to bound the regret by the internal regret, which is bounded by calibration error, which itself is bounded by $\epsilon$. 
The two main issues with their approach are first, the additional cost of running $1/\epsilon$ calibration algorithms in parallel; and second, having to rely on the calibration error bound in order to bound the regret.
Our technique bypasses these constraints 
% RDK delete: "completely" 
by appealing to Blackwell’s Approachability Theorem. With Blackwell’s Approachability Theorem, we can treat this problem as a vector-valued game where one tries to simultaneously minimize the calibration
% RDK edit: component of the vector and the regret component of the vector. 
and regret components of the vector. Instead of having $1/\epsilon$ different calibration algorithms, we have only a single calibration algorithm which also takes regret into account. The single calibration algorithm achieves a stronger guarantee by leveraging the fact that proper scoring rules incentivize calibration.
We also take this a step further by giving precise error bounds as a function of the time horizon, and allowing a trade-off between calibration error and regret. 

\paragraph{Online Minimax multiobjective optimization}
\rdkedit{%
An even more general problem than recalibration is online multiobjective optimization, for which \citet{multicalibeating} present a minimax theorem and a multiplicative-weights algorithm that achieves (a suitable notion of) the minimax value plus a sublinear regret term. By casting recalibration as an online multiobjective optimization problem, we show in
Appendix~\ref{sec:minimax-appendix} how to achieve 
calibration error and average regret both bounded by $O(\varepsilon + 1/\sqrt{\varepsilon T})$, for any
$\varepsilon > 0$. This matches the calibration error
bound for our Algorithm~\ref{alg:recalib}, 
but with a worse dependence 
on $\varepsilon$ in the regret bound. The reduction
from recalibration to online multiobjective optimization
uses loss vectors of dimension roughly $2^{1/\varepsilon}$,
so a na\"{i}ve implementation of the algorithm of
\citet{multicalibeating} would be computationally 
inefficient. In the Appendix, we 
indicate how it can be implemented to run in time 
$\mathrm{poly}(1/\varepsilon)$ per iteration 
by exploiting the special
structure of the loss vectors arising from our 
reduction. This running time is exponentially
faster than the na\"{i}ve reduction, but still
exponentially slower than the $O(\log(1/\varepsilon))$
running time per iteration of our
Algorithm~\ref{alg:recalib}.
}

\paragraph{Calibeating} \label{pg: calibeating}
Another closely related result is contained in a preprint by \citet{calibeat}. In their paper on “calibeating,” they present a method for transforming expert predictions to calibrated predictions, while measuring accuracy against an even more strict benchmark than ours: they compare the algorithm's loss to that of the expert after the calibration error has been removed, a benchmark called the ``refinement score''. They prove this for the loss function known as the Brier score, % and their proof technique relies on properties of the Brier score. 
when calibration is quantified using the $\ell_2$ objective.
\rdkedit{Our result is incomparable to theirs: while their benchmark for accuracy is stricter than ours, our quantification of calibration (using $\ell_1$ rather than $\ell_2$) is stricter than theirs. Furthermore, our recalibration procedure applies to any strictly proper scoring rule loss, whereas their calibeating procedure is specialized to the Brier score.}

%% file: sections/background.tex
\section{Background}
\subsection{Calibration}
Let $y_1, y_2, \ldots \in \{0,1\}$ be a sequence of outcomes, and $p_1, p_2, \ldots \in [0,1]$ a sequence of probability predictions by a forecaster. We define for every $T$ and every 
% probability interval $[a,b]$, where $0 \leq a \leq b \leq 1$, 
pair $p,\epsilon$ where $0 \leq p \leq 1$ and $\epsilon > 0$,
the quantities
\begin{align*}
	n_T(p,\epsilon) &:= \sum_{t=1}^T \I [p_t \in (p-\epsilon/2,p+\epsilon/2)], \\
	\rho_T(p,\epsilon) &:= \frac{\sum_{t=1}^T y_t \I[p_t \in (p-\epsilon/2,p+\epsilon/2)]}
	{n_T(p,\epsilon)}.
\end{align*}
The quantity $\rho_T(p-\epsilon/2,p+\epsilon/2)$ should be interpreted as the empirical frequency of $y_t = 1$, up to round $T$, on only those rounds where the forecaster's prediction was ``roughly'' equal to $p$. The goal of calibration, of course, is to have this empirical frequency $\rho_T(p,\epsilon)$ be close to the estimated frequency $p$. To capture how close an algorithm $\A$ to being $\epsilon$-calibrated, we use a notion of rate below.

\begin{definition}
% RDK edit: The definition as written was specific to $p=1$. I attempted to change this to general $p$. Did I get it correct?
%
% There was also a yucky issue regarding $i \epsilon$
% versus $i \epsilon + \epsilon/2$.
\newcommand{\probset}{\mathcal{P}}
    Let $\probset(\epsilon)$ denote the set of 
    midpoints of
    the intervals $[i \epsilon, (i+1) \epsilon]$
    for $i=0,1,\ldots,\lfloor \epsilon^{-1} \rfloor.$
  Let the $(\ell_1,\epsilon)$-calibration rate for forecaster $\A$ be

\begin{equation}
      C_{T}^{\epsilon}(\A) = \max \left\{ 0, \; \; \frac1T
      \left( \sum_{z \in \probset(\epsilon)} n_T(z,\epsilon) \cdot \left|  z - \rho_T(z,\epsilon) \right| \right)  - \frac{\varepsilon}{2} \right\}
\end{equation}
    We say that a forecaster is \emph{($\ell_1,\epsilon)$-calibrated} if $C_{T}^{\epsilon}(\A) = o(1)$. This in turn implies $\mathop{\lim\sup}_{T \to \infty} C_{T}^{\epsilon}(\A) = 0$.
  \end{definition}

\subsection{Proper Scoring Rules, Regret, and Recalibration}\label{sec:scoring}
\cite{kuleshov} define the problem of online 
recalibration in which the task is to
transform a sequence of uncalibrated forecasts $q_t$ into predictions $p_t$ that are calibrated and almost as accurate as the original $q_t$. They show that this objective is achievable if and 
only if the loss function used to measure forecast accuracy is 
a {\em proper scoring rule}, a term which we now define.

 Suppose there is a future event denoted by a random variable $X$ with a finite set $\Y$ of possible outcomes. For example: $\Y = \{ \text{rain}, \text{no rain} \}$. Let $\Delta_{\Y}$ be the set of probability distributions on $\Y$. An algorithm reports a probability distribution $p \in \Delta_{\Y}$, observes the outcome $y \in \Y$ and receives a score $S(p, y)$.
\begin{definition}
 A scoring rule is a function $S : \Delta_{\Y} \times \Y \mapsto \RR$. It is proper if accurately reporting the distribution of $X$ minimizes the expected score: that is, for all distributions $p,q \in \Delta_{\Y}$
 \begin{equation} \label{eq:scoring-rule}
     \E_{X \sim p} \left[ S(p,X) \right] \le
     \E_{X \sim p} \left[ S(q,X) \right] .
 \end{equation}
 Scoring rule $S$ is {\em strictly proper} if 
 Inequality~\eqref{eq:scoring-rule} is strict
 whenever $p \neq q$. 
\end{definition}
 Note that we adopt the convention that the scoring rule is a loss function rather than a payoff function, i.e.~$p$ is the unique probability that minimizes $S(\cdot,p)$ rather than maximizing it. 
    We extend $S$ to the domain $\Delta_{\Y} \times \Delta_{\Y}$
 by making it linear in the second variable. In other words,
 $S(q,p)$ is shorthand for $\E_{X \sim p} \left[ S(q,X) \right].$
 We assume the scoring rule $S$ is Lipschitz-continuous in its
 first variable, with Lipschitz constant $L_S$, i.e.
 \[
    \forall p,q \in \Delta_{\Y} \;
    \forall y \in \Y \qquad
    |S(p,y) - S(q,y)| \le L_S \cdot \| p-q \| ,
\]
where $\|p-q\|$ denotes the total variation distance
between $p$ and $q$.
 % RDK moved: \citet{kuleshov} show that recalibration can only be guaranteed when the loss function of the oracle is a proper scoring rule. Thus, our

We measure a forecaster's accuracy by 
comparing with the score of the oracle. Let $q_1, q_2, \ldots \in [0,1]$  be a sequence of probability predictions by an oracle.
\begin{definition} \label{def:regret}
  Let the regret at timestep $t$ for forecaster $\A$ be
    \[
      r(p_t, q_t, y_t) = S(p_t, y_t) - S(q_t, y_t)
    \]
    This leads to an \emph{average regret} of $R_T(\A) = \frac{1}{T} \sum_{t=1}^T r(p_t, q_t, y_t)$. We say that a forecaster has \emph{no-regret} if $R_T(\A) = o(1)$. This in turn implies $\mathop{\lim\sup}_{T \to \infty} R_T (\A) = 0$. We also say a forecaster has \emph{$\delta$-regret rate} if $R_T(\A) \leq \delta$.
\end{definition}
% \subsection{Recalibration}

% Specifically, they say an algorithm $\A$ is an $(\epsilon, \ell_p)$-accurate online recalibration algorithm for the loss $\ell^\text{acc}$ if the forecasts $p_t = \A(p_t^F)$ are $(\epsilon, \ell_p)$-calibrated and the regret with respect to $p_t^F$ is small in terms of $\ell^\text{acc}$. That is, $\lim \sup_{T \rightarrow \infty} \frac{1}{T} \sum_{t=1}^T \left( \ell^\text{acc}( y_t, p_t) - \ell^\text{acc}( y_t, p_t^F)\right) \leq \epsilon$. For our setup, we make a few changes. We separate the accuracy from the amount of calibration. That is, we want to allow for $\epsilon$-calibration while keeping regret less than $\delta$. Similar to calibration, we define a notion of rate for online recalibration as follows:

\begin{definition}\label{def:rate}
Let the $(\ell_1,\epsilon, \delta)$-recalibration rate for forecaster $\A$ be
\begin{align}
		C_{T}^{\epsilon, \delta}(\A) = \max \left\{0, C_{T}^{\epsilon}(\A), R_T(\A) - \frac{\delta}{2}\right\}
\end{align}
	We say that a forecaster is \pcoedit{\emph{($\ell_1,\epsilon, \delta)$-recalibrated}} if \pcoedit{$C_{T}^{\epsilon, \delta}(\A) = o(1)$}. This in turn implies \pcoedit{$\mathop{\lim\sup}_{T \to \infty} C_{T}^{\epsilon, \delta}(\A) = 0$}. 
\end{definition}
This definition is analogous to Definition 4 
in~\cite{kuleshov}, except that we have quantified
the calibration and accuracy using two parameters,
$\epsilon$ and $\delta$, whereas they use $\epsilon$
for both.

\subsection{Blackwell's Approachability Theorem}
Blackwell approachability \citep{blackwell} generalizes the problem of playing a repeated two-player zero-sum game to games whose payoffs are vectors instead of scalars. In a Blackwell approachability game, at all times $t$, two players interact in this order: first, Player 1 selects an action $x_t \in X$; then, Player 2 selects an action $y_t \in Y$; finally, Player 1 incurs the vector-valued payoff $u(x_t, y_t) \in \RR^d$. The sets $X , Y$ of player actions are assumed to be compact convex subsets of finite-dimensional vector spaces, and $u$ is assumed to be a biaffine function on $X \times Y$. Player 1's objective is to guarantee that the average payoff converges to some desired closed convex target set $\mathcal{S} \subseteq \RR^d$. Formally, given target set $\mathcal{S} \subseteq \RR^d$, Player 1's goal is to pick actions $x_1, x_2, \ldots \in X$ such that no matter the actions $y_1, y_2, \ldots \in Y$ played by Player 2,
\begin{equation} \label{eq:approach}
  \dist \left(\frac{1}{T} \sum_{t=1}^T u(x_t, y_t), \mathcal{S} \right)  \rightarrow 0 \quad \text{as} \quad T \rightarrow \infty
\end{equation}
The action $x_t$ is allowed to depend on the realized payoff
vectors $u_s(x_s,y_s)$ for $s=1,2,\ldots,t-1$.
We say the set $S$ is approachable if Player~1
has a strategy that attains the goal~\eqref{eq:approach}
no matter how Player 2 plays. Blackwell's Approachability
Theorem asserts that a convex set $\mathcal{S} \subset \reals^d$
is approachable if and only if every closed halfspace
containing $\mathcal{S}$ is approachable.
Henceforth we refer to 
this necessary and sufficient condition as 
{\em halfspace-approachability}.

In this paper, we shall adopt the notation, $\dist_p(x, \mathcal{S})$ to be the $\min_{s \in S} \left\| x - s\right\|_p$. We will refer to the $\ell_p$ ball $\in R^d$ of radius $r$ centered at the origin as $B_p^{d}(r)$.
\pcodelete{
Due to the nature of our vector payoff formulation, we will also use a compounded notation of distance: when $x = (x_1, x_2) \in \RR^d \times \RR$ and 
$\mathcal{S} = \mathcal{S}_1 \times \mathcal{S}_2 \subseteq \RR^d \times \RR$, we will write $\dist_p^r(x, \mathcal{S})$ to denote $||(\dist_p(x_1, \mathcal{S}_1), \dist_1(x_2, \mathcal{S}_2)) ||_r$. We will refer to the $\ell_p$ ball $\in R^d$ of radius $r$ centered at the origin as $B_p^{d}(r)$.
To specify the target convex set for our Blackwell's instance, we shall refer to the following set definitions.
\begin{definition} \label{def:sets}
  For a fixed $\epsilon, \delta, m > 0$ and distance metric $\ell_p$, let 
  \begin{equation}
    \mathcal{S}_p^m (\epsilon, \delta) =  \left\{  (x, z) \mid \ x \in \RR^{m+1}, z \in \RR \quad \text{s.t} \quad \norm{x}_p \leq \frac{\epsilon}{2}, \ z \leq \frac{\delta}{2} \right\}
  \end{equation}
  let 
  \begin{equation}
    \K_q^m =  \left\{  (a, b) \mid \ a \in \RR^{m+1}, b \in \RR \quad \text{s.t} \quad \norm{a}_q \leq 1, \ 0 \leq b \leq 1 \right\}
  \end{equation}
  Observe that $\mathcal{S}_p^m (\epsilon, \delta)$ can be thought of as $B_p^{m+1}(\epsilon/2) \oplus \left( -\infty, \frac{\delta}{2} \right]$. Similarly, $\K_q^m$ can be expressed as $B_q^{m+1}(1) \oplus \left[ 0, 1 \right]$
  \end{definition}
}

We now give an equivalent and alternative characterization of the definition of recalibration rate (Definition \ref{def:rate}): let the {\em recalibration vector} at time $T$ denoted $\mathbf{v}_T$ be given by: $\mathbf{v}_T = \mathbf{c}_T \ \oplus \ R_T$ where
$\mathbf{c}_T(i) = \frac{n_T(i\epsilon,\epsilon)}{T}\left(  i \epsilon - \rho_T(i\epsilon,\epsilon) \right)$ 
for $0 \le i \le \lceil \epsilon^{-1} \rceil,$ \text{and} $R_T = \frac{1}{T} \sum_{t=0}^T S(p_t, y_t) - S(q_t, y_t).$
\begin{lemma}\label{lem:rateisapproach}
\pcoedit{
\begin{equation}
  C_{T}^{\epsilon, \delta}(\A) = \max \left\{ \dist_1 \left( \mathbf{c}_T, B_1^{\epsilon^{-1}} (\epsilon/2) \right), R_t - \delta/2 \right\}
\end{equation}
}
\end{lemma}  
%\begin{proof}
%\begin{align}
%  \dist_p^\infty \left( \vb_T, \mathcal{S}_p (\epsilon, \delta) \right) = \max \left\{  \dist_p \left( \cbb_T, B_p^{m+1}(\epsilon/2) \right), \dist_1 \left( R_T, \left( -\infty, \frac{\delta}{2} \right] \right) \right\}
%\end{align}
%\end{proof}

%% file: sections/approachability.tex
\section{Recalibration via Approachability}
\label{sec:approachability}
We now describe the construction of the payoff game that allows us to reduce recalibration to approachability. This payoff game modifies the standard construction for calibration in \citep{foster_proof_1999, pmlr-v19-abernethy11b} by adding an additional dimension for regret. 

\subsection{Reduction}
For any $m \geq \sqrt{4L_s}$ where $L_s$ is the lipschitz constant of the scoring rule, we will show how to construct an $(\ell_1,\epsilon, \delta)$-recalibrated forecaster for $\epsilon = \frac{1}{m}$ and $\delta = \frac{4L_s}{m^2}$. On each round $t$, after observing the oracle's prediction $q_t$, a forecaster will randomly predict a probability $p_t \in \{0/m, 1/m, 2/m, \ldots, (m-1)/m, 1\}$, according to the distribution $\wv_t$, that is $\text{Pr}(p_t = i/m) = w_t(i)$. We define a vector-valued game. Let the player choose $\wv_t \in \X := \Delta_{m+1}$, and the adversary choose $y_t \in \Y := [0,1]$, and the payoff vector will be $\lb_t(\wv_t,y_t) = \cbb (\wv_t,y_t) \oplus r(\wv_t,q_t, y_t)$ \footnote{$\oplus$ represents concatenation} 
% \footnote{Blackwell's original theorem only applies to fixed losses. Our reduction remains valid because \citet{pmlr-v19-abernethy11b}'s proof applies to changing losses.} 
defined as follows:
%\rdkmargincomment{Changed angle brackets to round parentheses on the right side of \eqref{eq:calib_game} to avoid confusion with inner product notation.}
\begin{align}\label{eq:calib_game}
&\cbb (\wv_t,y_t) := \\
 &\left( \wv_t(0)\left(y_t - \frac 0 m\right), \wv_t(1)\left(y_t - \frac 1 m\right), \ldots, \wv_t(m)(y_t - 1) \right)
\end{align}
\begin{align}\label{eq:regret_game}
r(\wv_t, q_t, y_t) := \sum_{i = 0}^m \wv_t(i) \left(S \left(\frac i m, y_t \right) - S(q_t, y_t)\right)
\end{align}
The set we wish to approach is 

\begin{align}\label{eq:approach_set}
   &\approachset \\
   &= \left\{  (x, z) \mid \ x \in \RR^{m+1}, z \in \RR \ \text{s.t} \ \norm{x}_1 \leq \frac{1}{m}, \ z \leq \frac{4L_s}{m^2} \right\}
\end{align}

% \pcocomment{Include paragraph here explaining how approachability is able to achieve an improved $\frac{1}{m^2}$
% Kuleshov can be interpreted as using a quadratic number of dimensions in the payoff vector and we show how to achieve the same low regret guarantee using a linear number of dimensions. This dimension reduction from $m^2$ to $m$ dimension is not for free and we have to prove the improvement directly. The main technical contribution is that we do that hardwork of showing approachability for a lower dimensional problem while they get approachability for free from a previous result by Blum? showing calibration can be reduced to internal regret minimization}
In \rdkedit{Section~\ref{pg:recalibration}}, we pointed out that \cite{kuleshov}'s approach works by running $1/\epsilon = m$ many calibration algorithms in parallel, one for each prediction interval $[i/m, (i+1)/m]$. Each calibration algorithm solves a vector-valued game with payoff vectors of dimension $m$. Thus, their approach can also be interpreted as using a quadratic number of dimensions $(m^2)$ in the payoff vector while we show how to achieve the same low regret guarantee using a linear number of dimensions $(m+2)$. Our main technical contribution is that 
the lower-dimensional problem we formulate requires a novel proof of approachability, which our work supplies, whereas in
the higher-dimensional problem formulated implicitly by \cite{kuleshov} approachability follows ``for free'' due to a more general result by \citep{blum07a,cesa2006prediction}.

\subsection{Proof of Approachability}
For the calibration vector-payoff game, \cite{pmlr-v19-abernethy11b} prove approachability via response-satisfiability. While this is arguably the simplest way to prove approachability, it is important to note that to construct an algorithm for approaching the desired set, simply proving response-satisfiability is not enough. A halfspace oracle needs to be provided as well. Although \cite{pmlr-v19-abernethy11b} prove approachability by response-satisfiability, they present a halfspace oracle based on the construction in Foster’s halfspace-approachability proof. For our recalibration problem, we prove approachability by showing halfspace-approachability. Our proof is constructive, hence it directly yields a halfspace oracle.

\begin{theorem} \label{thm:approach}
For the vector-valued game defined in \ref{eq:calib_game}, the set $\mathcal{S} = \approachset$ is approachable. That is, any halfspace $H$ containing $\mathcal{S}$ is approachable.
\end{theorem}
\begin{proof}
First we characterize the set of halfspaces containing $S$. Let $H$ be a halfspace of $\RR^{m+2}$ defined by the equation $\langle a, x \rangle + bz \leq \theta$ for $x \in \RR^{m+1}, z \in \RR$. We claim that $\mathcal{S} \subseteq \mathcal{H}$ iff $b \geq 0$ and $\theta \geq \left( \frac{\norm{a}_\infty}{m} + \frac{4bL_s}{m^2} \right)$.
To see this, observe that for $\mathcal{H}$ to contain $S$. It must be the case that  
\[
\max \left\{ \langle a, x \rangle + bz  \mid \ \norm{x}_1 \leq \frac{1}{m}, z \leq \frac{4L_s}{m^2} \right\} \leq \theta
\]
First, we need $b \geq 0$, since we can choose $z$ to violate this constraint otherwise. Secondly, we need $\theta \geq \left( \frac{\norm{a}_\infty}{m} + \frac{4bL_s}{m^2} \right)$, since we can choose $x$ and $z$ to violate this constraint otherwise. Thus, if $\mathcal{S} \subseteq H$, then both conditions $b \geq 0$ and $\theta \geq \left( \frac{\norm{a}_\infty}{m} + \frac{4bL_s}{m^2} \right)$ must hold for $H$. 
Conversely, if both conditions $b \geq 0$ and $\theta \geq \left( \frac{\norm{a}_\infty}{m} + \frac{4bL_s}{m^2} \right)$ hold for $H$, then $\mathcal{S} \subseteq H$. This is because for any $(x, z) \in \mathcal{S}$, $\langle a, x \rangle + bz \leq \left( \frac{\norm{a}_\infty}{m} + \frac{4bL_s}{m^2} \right) \leq \theta$ and if $b < 0 $, we can obtain a contradiction by choosing $z < -\frac{\theta}{b}$.

WLOG, we will assume $\theta = \left( \frac{\norm{a}_\infty}{m} + \frac{4bL_s}{m^2} \right)$, since approachability of a halfspace defined by $\langle a, x \rangle + bz \leq \left( \frac{\norm{a}_\infty}{m} + \frac{4bL_s}{m^2} \right)$ implies approachability of  $\langle a, x \rangle + bz \leq \theta$ for $\theta \geq \left( \frac{\norm{a}_\infty}{m} + \frac{4bL_s}{m^2} \right)$. 
That is, we will only concern ourselves with proving halfspace-approachability for halfspaces such that $\theta = \left( \frac{\norm{a}_\infty}{m} + \frac{4bL_s}{m^2} \right)$. 
For a halfspace such that $a = \zero$, we follow the halfspace oracle in \ref{halfspace_equation} and set $a_i = 0$ for all $i$. This gives us regret at most $\frac{4 L_s}{m^2}$; see proof of \ref{lem:smallest_halfspace} in the appendix.
% , and thus guarantees that $bz = 0 \leq \frac{4L_s}{m^2}$}
If $a \neq \zero$, then we can consider the halfspace normalized by $\norm{a}_\infty$, that is, the halfspace defined by $a' = \frac{a}{\norm{a}_\infty}, b' = \frac{b}{\norm{a}_\infty}$ and $\theta = \frac{1}{m} + \frac{4b'L_s}{m^2}$. Since $\norm{a'}_\infty = 1$ and $b' \geq 0$, by Lemma \ref{lem:smallest_halfspace}, this halfspace is approachable. Consequently, any halfspace containing $\mathcal{S}$ is approachable.
\end{proof}
 
\begin{restatable}{lemma}{approachresult}\label{lem:smallest_halfspace}
% For a fixed $(a,b)$ such that \pcoedit{$\|a\|_{\infty} = 1$} and $b \geq 0$, the halfspace $H_1$ parameterized by $(a,b)$ defined by below is approachable
% \begin{equation}
%   H_1 := \left\{  (x, z) \mid \ \langle a, x \rangle + bz \leq \frac{1}{m} + \frac{4bL_s}{m^2}\right\}
% \end{equation}
Consider a pair $(a,b) \in \mathbb{R}^{m+1} \times \mathbb{R}$ such that $\|a\|_{\infty} = 1$ and $b \geq 0$. The halfspace $H_1$, defined below, is approachable.
\begin{equation}
H_1 := \left\{ (x, z) \in \mathbb{R}^{m+1} \times \mathbb{R} \ \bigg| \ \langle a, x \rangle + bz \leq \frac{1}{m} + \frac{4bL_s}{m^2} \right\}
\end{equation}
\end{restatable}

The full proof can be found in the appendix. We provide a proof sketch here. To show that $H_1$ is approachable, we will find a mixed distribution for the forecaster (i.e, a probability distribution over $p \in \{0/m, 1/m, 2/m, \ldots, (m-1)/m, 1\}$)	such that $\E_p \left[ \langle a, \cbb(p, y) \rangle + b r(p,q_t,y) \right] \leq \frac{1}{m} + \frac{4 b L_S}{m^2}$ for any $y \in \{0,1\}$. 
For simplicity, define  
\begin{align}\label{Ffunction}
    f (i, y) &= a_i \left(\frac{i}{m} - y \right) + b \left[ S \left(\frac{i}{m}, y \right) - S(q_t, y) \right]\\
    F_i &= \begin{bmatrix} f(i,0) \\ f(i,1) \end{bmatrix}
\end{align}
so our objective becomes to show that there exists a distribution $p$ over $\frac{i}{m} \in \{0, \ldots, m \}$ such that $\En_p f(i, y) \leq \frac{1}{m} + \frac{4 b L_S}{m^2}$ for $y \in \{0,1\}$, or equivalently that the vector $\En_p F_i$ belongs to the quadrant-shaped set $(-\infty,\frac1m + \frac{4 b L_S}{m^2}] \times (-\infty,\frac1m + \frac{4 b L_S}{m^2}]$. 
We will be choosing $p$ to be either a point-mass on 
$\frac{i}{m}$ for some $i$, or a distribution on two
consecutive values in the set $\{0,\frac1m, \frac2m, \ldots, 1\}.$ 
For $p \in [0,1]$ let $D(p)$ denote the vector corresponding to the scoring rule term in $F_i$. 
\[
    D(p) = b \cdot \begin{bmatrix}
        S(p,0) - S(q_t,0) \\ S(p,1) - S(q_t,1)
    \end{bmatrix}
\]
As a result of the fact that $S$ is a proper scoring rule, an important observation is that the curve formed by $D(p)$ is convex and its tangent lines are parallel to $\begin{bmatrix} p \\ p - 1 \end{bmatrix}$. Thus, $F_0,F_1,\ldots,F_m$ are points on a sequence of tangent lines to the convex curve formed by $D(p)$. Additionally, we can show that $F_0$ lies in the left half-plane while $F_m$ must belong to the lower half-plane. Thus, $F_0,F_1,\ldots,F_m$ are always in the second, third or fourth quadrants and lie on lines with slopes that are slowly changing from negative to positive. 
\begin{center}
\includegraphics[scale=0.4]{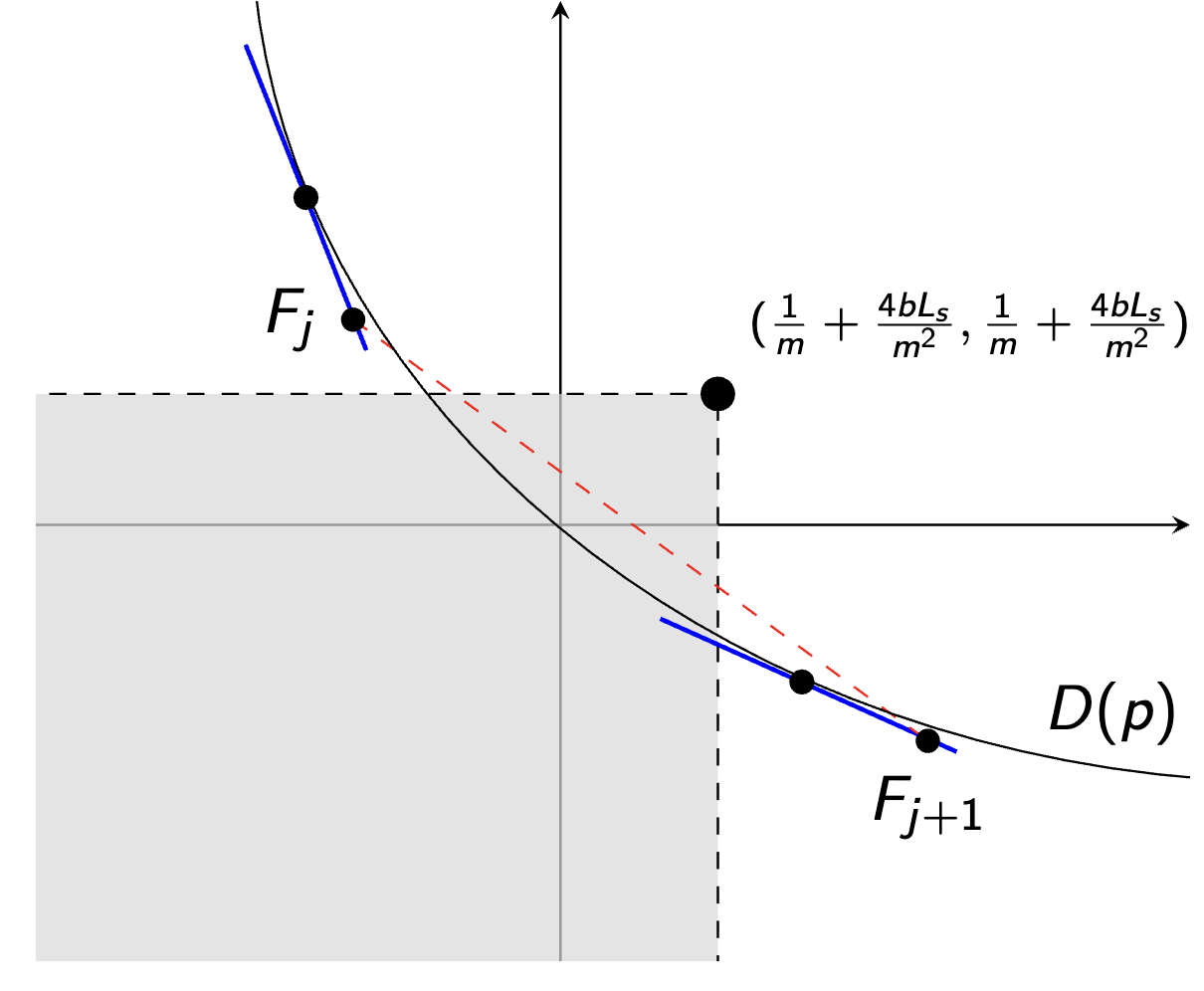}
\end{center}
If $F_i$ belongs to the third quadrant --- that is, the set $(-\infty,0] \times (-\infty,0]$ --- then we choose $p$ to be a point-mass on $i$. This guarantees that $\E_p F_i \leq \frac1m + \frac{4 b L_S}{m^2}$. Otherwise, there must be at least one index $j$ such that $F_j$ lies in the second quadrant while $F_{j+1}$ lies in the fourth quadrant. Using plane geometry, we show that the line segment joining $F_j$ and $F_{j+1}$ intersects
the set $(-\infty,\frac1m + \frac{4 b L_S}{m^2}] \times (-\infty,\frac1m + \frac{4 b L_S}{m^2}]$ as required. The rest of the proof can be found in the appendix.

\subsection{Efficient Algorithm via Online Linear Optimization}\label{sec:algorithm}
We now show how the results in the previous section lead to an efficient algorithm for online recalibration. The steps in this section are parallel to those in Section 5.2 of \cite{pmlr-v19-abernethy11b} but we have to repeat them because our payoff game and convex sets are different.

\begin{theorem} \label{thm:algorithm}
For any $m$, there exists a $(\ell_1,\frac{1}{m}, \frac{4L_s}{m^2})$-online recalibration algorithm that runs in time $O(\log m)$ per iteration and guarantees a recalibration rate of $O\left( \sqrt{\frac{m}{T}}\right) $
\end{theorem}

Following the steps of the reduction from Approachability to OLO outlined in \cite{pmlr-v19-abernethy11b}, we provide a convex set $\K$ whose elements correspond to halfspaces containing $\approachset$,
and express the distance of a loss vector to the set $\mathcal{S}$ we wish to approach as an optimization over the convex set $\K$. We do so in Lemma~\ref{lem:dualsetopt}. Then, we present an algorithm (halfspace oracle) such that given a halfspace $\thb_t \in \K$, it returns a distribution $\wv_t \in \Delta_{m+1}$ with the guarantee that $\langle \lb_t (\wv_t,y_t), \thb_t \rangle \leq \frac{1}{m} + \frac{4L_s}{m^2}$. Lastly, we present an algorithm for recalibration that uses Online Gradient Descent \cite{zinkevich_online_2003} to select the halfspace $\thb_t \in \K$ to approach at each timestep.
% \begin{enumerate}
% 	\item A convex set $\K$
% 	\item An efficient learning algorithm $\A$ which, for any sequence $\fv_1, \fv_2, \ldots$, can select a sequence of points $\thb_1, \thb_2, \ldots \in \K$ with the guarantee that $\sum_{t=1}^T \langle \fv_t, \thb_t \rangle - \min_{\thb \in \K} \sum_{t=1}^T \langle \fv_t, \thb \rangle = o(T)$. 
% %For the reduction, we shall set $\fv_t \leftarrow -\lb(\wv_t,y_t)$.
% 	\item An efficient oracle that can select a particular $\wv_t \in \X$ for each $\thb_t \in \K$ with the guarantee that
% 	\begin{equation} \label{eq:dist_vs_reg}
% 		\dist_p^1\left(\frac 1 T \sum_{t=1}^T \lb_t(\wv_t,y_t), \mathcal{S} \right) \leq \frac 1 T \left( \sum_{t=1}^T \langle -\lb_t(\wv_t,y_t), \thb_t \rangle - \min_{\thb \in \K} \sum_{t=1}^T \langle -\lb_t(\wv_t,y_t), \thb \rangle\right)
% 	\end{equation}
% \end{enumerate}

%\paragraph{Convex Set}
We define the convex set $\K$ as follows
\begin{equation}
  \K := \left\{  (a, b) \mid \ a \in \RR^{m+1}, b \in \RR \ \text{s.t} \ \norm{a}_\infty \leq 1, \ 0 \leq b \leq 1 \right\}
\end{equation}
This is an appropriate choice of $\K$ due to Lemma \ref{lem:dualsetopt}, since it allows us to upper bound the distance to $S$ in terms of a linear optimization objective over the set $\K$.
\begin{restatable}{lemma}{dualsetopt}\label{lem:dualsetopt}
For any vector $\xb \in \RR^{m+2}$ such that $\| \xb_{1:m+1}\|_{1} \geq 1 / m$, and $| \xb_{m+2} | \geq \frac{4L_s}{m^2}$, 
%not in $\mathcal{S}_p^m(\frac{1}{m}, \frac{L_s}{m^2})$, 
\begin{equation}
  \pcoedit{\dist_1 \left( \xb, \approachset \right)} = -\frac{1}{m} - \frac{4L_s}{m^2} - \min_{\theta \in \K} \langle - \xb, \theta \rangle
\end{equation} \pcodelete{where $q$ is such that $\frac{1}{p} + \frac{1}{q} = 1$}
\end{restatable}
We defer the proof of Lemma~\ref{lem:dualsetopt} to the appendix. The usefulness of the lemma above is that it allows us to combine the approachability guarantee of \rdkedit{Theorem}~\ref{thm:approach} to upper bound the distance to the target convex set in terms of regret of an online linear optimization algorithm.
\begin{align}\label{eq:dist-to-olo}
  &\dist_1 \left( \frac{1}{T} \sum_{t=1}^T \lb_t (\wv_t, y_t), \approachset \right) \\
  &= -\frac{1}{m} - \frac{4L_s}{m^2} - \min_{\theta \in \K} \left\langle - \frac{1}{T} \sum_{t=1}^T \lb_t (w_t, y_t), \thb \right\rangle \\
  &\leq \frac 1 T \left( \sum_{t=1}^T \langle -\lb_t(\wv_t,y_t), \thb_t \rangle - \min_{\thb \in \K} \sum_{t=1}^T \langle -\lb_t(\wv_t,y_t), \thb \rangle\right) 
\end{align}
where the inequality follows from the approachability guarantee of \rdkedit{Theorem}~\ref{thm:approach}: for any halfspace $\theta_t$, there exists a distribution $\wv_t$ such that $\langle \lb_t (\wv_t, y_t), \theta_t \rangle \leq \frac{1}{m} + \frac{4L_s}{m^2}$ for any $y_t \in \{ 0,1 \}.$

\paragraph{The Halfspace Oracle: $\text{Approach}(\thb_{t+1})$}
Given any $\thb_t \in \K$, we must construct $\wv \in \Delta_{m+1}$ so that $\langle \lb_t (\wv_t,y_t), \thb_t \rangle \leq \frac{1}{m} + \frac{4L_s}{m^2}$ for any $y_t$. The proof of approachability for Lemma~\ref{lem:smallest_halfspace} is a constructive one and describes how to choose $\wv_t \in \Delta_{m+1}$ given $\thb_t$. Recall functions $f(i,y)$ and $F_i$ defined in \ref{Ffunction}. The algorithm firsts check if $F_0$ or $F_m$ is in the 3rd quadrant. If one of them is, then we output a point distribution at the corresponding probability value.
If none of $F_0$ or $F_m$ is in the 3rd quadrant, then we binary search for an index $i$ with $F_i$ in the 3rd quadrant or a pair of consecutive indices $j, j+1$ where $F_j$ is in 2nd quadrant and $F_{j+1}$ is in the 4th quadrant. In the first case, $wv_t(i) = 1$ and 0 everywhere else. In the second case, we set 
\begin{align}\label{halfspace_equation}
    \wv_t(j) = \frac{f(j+1,1) - f(j+1, 0)}{f(j, 0) - f(j+1, 0) - f(j,1) + f(j+1,1)} \\
    \wv_t(j+1) = \frac{f(j,0) - f(j, 1)}{f(j, 0) - f(j+1, 0) - f(j,1) + f(j+1,1)}
\end{align}
and 0 everywhere else. 
The correctness of this procedure follows from the proof of Lemma~\ref{lem:smallest_halfspace}. 
Note that $F_i$ does not need to be pre-computed for every index. It can be computed online during the binary search steps. Thus, this halfspace oracle can be implemented in $O(\log m)$ steps.

\paragraph{The Learning Algorithm: $\text{OGD}(\thb_t | l_t)$}
%There are lots of available no-regret algorithms we could use here. 
Similar to \citet{pmlr-v19-abernethy11b}, we use the Online Gradient Descent algorithm \citep{zinkevich_online_2003} as the learning algorithm. 
% As \citet{pmlr-v19-abernethy11b} points out in the calibration version, the vectors $\lb_t (\wv_t, y_t)$ are sparse and have at most two nonzero coordinates. This is also true for our recalibration payoff game since $\wv_t$ is a distribution over two consecutive probabilities values. Hence, the Gradient Descent Step requires only $O(1)$ computation and the Projection Step can be performed $O(1)$ (for $p = 1$). Since $\thb$ is the only state the OGD algorithm needs to store, the storage space required is $O(\min\{T, m\})$.
\begin{algorithm}[H]
	\caption{Online Recalibration Algorithm} \label{alg:recalib}
	\begin{algorithmic}
		\STATE Input: some natural number $m \geq \sqrt{4L_s}$
		\STATE Initialize: $\thb_1 = \mathbf{0}, \wv_1 \in \Delta_{m+1}$
		\FOR{$t=1, \ldots, T$}
    \STATE Observe $q_t$ from black-box prediction oracle
		\STATE Sample $i_t \sim \wv_t$, predict $p_t = \frac{i_t}{m}$, observe $y_t$ 
    \STATE Set $l_t := -\lb_t (\wv_t, y_t)$ 
		\STATE Query learning algorithm: $\thb_{t+1} \leftarrow \text{OGD}(\thb_t | l_t)$ \quad  // Online Gradient Descent step
		\STATE Query halfspace oracle: $\wv_{t+1} \leftarrow \text{Approach}(\thb_{t+1})$ \ // Obtain $\wv_{t+1} \in \Delta_{m+1}$ from $\thb_{t+1}$
		\ENDFOR
	\end{algorithmic}
\end{algorithm}

OGD guarantees that the regret is no more than $DG\sqrt{T}$ where $D$ is the $\ell_2$ diameter of the set and $G$ is the $\ell_2$-norm of the largest cost vector. For the convex set $\K$, the $\ell_2$ diameter is $O(\sqrt{m})$. The $\ell_2$-norm of the calibration component of the vector is bounded by $\sqrt{2}$. To make the size of the regret at time $t$ small and at most 1, we normalize by the lipschitz-constant $L_s$
\pcoedit{
\begin{align}
  C_{T}^{\epsilon, \delta}(\A) &\leq \dist_1 \left( \frac{1}{T} \sum_{t=1}^T \lb_t(\wv_t, y_t), \approachset \right) \\
  &\leq \frac{\text{Regret}_t}{T} 
  \leq \frac{GD}{\sqrt{T}} 
  = O\left( \sqrt{\frac{m}{T}}\right)
\end{align}
}
\pcodelete{
\begin{theorem}[$\ell_p$ generalization of Theorem \ref{thm:algorithm}] \label{thm:lp_algorithm}
  For any $m$, there exists a $(\ell_p, \frac{1}{m}, \frac{4L_s}{m^2})$-online recalibration algorithm that guarantees a recalibration rate of $O\left( \frac{m^{c}}{\sqrt{T}}\right)$ where $c = \left(\frac{1}{2} - \frac{1}{q} \right)$ for $q > 2$ and $c = 2$ for $q \leq 2$
\end{theorem}

We prove this result in the appendix similar to how we show the $\ell_1$ version of the theorem. The convex set $\K_\infty^m$ is replaced with a more general $\K_q^m$ for $q$ such that $\frac{1}{p} + \frac{1}{q} = 1$. The results from Lemma~\ref{lem:dualsetopt} and \ref{eq:dist-to-olo} hold for any $p$. The halfspace oracle algorithm also applies for any $p$. We obtain our bounds using Online Gradient Descent. However, the learning algorithm can be chosen as Online Mirror Descent for a regularizer optimized for the corresponding convex set $\K_q^m$.
}

%% file: sections/rates.tex
\section{Convergence Rates}\label{sec:rates}
In this section, we describe how the results from the previous sections can be used to obtain bounds on calibration error and regret. 

\begin{theorem}\label{thm:tradeoff}
For any $x \in [\frac13, \frac25]$, given a black-box prediction oracle, there exists a forecasting algorithm that simultaneously achieves expected regret $O(T^{-x})$ while keeping $\ell_1$-calibration error less than $T^{2x-1}$.
\end{theorem}
\begin{proof}
In Theorem \ref{thm:algorithm}, we show that for any $m$, there exists an $(\ell_1,\frac{1}{m}, \frac{4L_s}{m^2})$-online recalibration algorithm which satisfies a recalibration rate of  $O\left(\sqrt{\frac{m}{T}}\right)$. By definition \ref{def:rate}, this implies that the $\ell_1$-calibration error is upper bounded by $O(\frac{1}{m} + \sqrt{\frac{m}{T}})$ and the regret is upper bounded by $O(\frac{1}{m^2} + \sqrt{\frac{m}{T}})$. Setting $m = T^{1-2x}$, we obtain an algorithm that guarantees regret of $O(T^{-x})$ and calibration error $O(T^{2x-1})$
\end{proof}
\pcodelete{
\pcocomment{upgraded the graph to a region of achievable rates instead of just a line}
\begin{figure}[H]
  % \caption{The graph below captures the linear tradeoff between regret and $\ell_1$ calibration error.}
  % \centering
% \begin{tikzpicture}
%   \begin{axis}[
%       width=0.55*\textwidth, %% change the factor of 0.8 to change width in paper
%       axis lines = left,
%       xlabel = Regret,
%       ylabel = Calibration Error,
%       ymin=0,ymax=100,
%       xmin=-10,xmax=20,
%       ticks=none
%   ]
%   \addplot [black] coordinates { (-6,80) (6, 20) };
%   \addplot [black, dashed] coordinates { (-15, 80) (-6,80) };
%   \addplot [black, dashed] coordinates { (-6, -20) (-6,80) };
%   \addplot [black, dashed] coordinates { (-15, 20) (6,20) };
%   \addplot [black, dashed] coordinates { (6, -20) (6,20) };
%   \node[label={45:{$(T^{-\frac{2}{5}}, T^{-\frac{1}{5}})$}},circle,fill,inner sep=2pt] at (axis cs:-6,80) {};
%   \node[label={60:{$(T^{-\frac{1}{3}}, T^{-\frac{1}{3}})$}},circle,fill,inner sep=2pt] at (axis cs:6,20) {};
%   \end{axis}
%   \end{tikzpicture}
\caption{The graph below captures the linear tradeoff between regret and $\ell_1$ calibration error. According to Theorem~\ref{thm:tradeoff}, the set of jointly achievable rates contains the shaded region.}
  \centering
\begin{tikzpicture}
\begin{axis}[
    width=0.8*\textwidth, %% change the factor of 0.8 to change width in paper
    axis lines = left,
    xlabel = Regret,
    ylabel = Calibration Error,
    ymin=0,ymax=130,
    xmin=0,xmax=100,
    ticks=none
]
\addplot+[name path=A, black, thick] coordinates { (10,80) (45, 20) };
\addplot [black, dashed] coordinates { (10, -20) (10,110) };
\addplot+[name path=BOT, black, dashed] coordinates { (-10, 20) (120,20) };
\addplot+[name path=TOP, gray] coordinates { (10, 110) (75,110) };
\addplot+[name path=RIGHT, gray] coordinates { (75, 20) (75,110) };

\addplot[gray!20] fill between[of=TOP and A, soft clip={domain=10:50}];
\addplot[gray!20] fill between[of=TOP and BOT, soft clip={domain=45:75}];

\node[label={30:{$(T^{-\frac{2}{5}}, T^{-\frac{1}{5}})$}},,circle,fill,inner sep=2pt] at (axis cs:10,80) {};
\node[label={60:{$(T^{-\frac{1}{3}}, T^{-\frac{1}{3}})$}},circle,fill,inner sep=2pt] at (axis cs:45,20) {};
\node[label={30:{$(T^0, T^{-\frac{1}{3}})$}},circle,fill,inner sep=2pt] at (axis cs:75, 20) {};
\node[label={30:{$(T^0, T^0)$}},circle,fill,inner sep=2pt] at (axis cs:75, 110) {};
\node[label={30:{$(T^{-\frac{2}{5}}, T^0)$}},circle,fill,inner sep=2pt] at (axis cs:10, 110) {};

\end{axis}
\end{tikzpicture}
\end{figure}
}
\pcodelete{
Theorem~\ref{thm:tradeoff} generalizes to 
$\ell_p$ calibration as follows.
\begin{theorem}[$\ell_p$ generalization of Theorem \ref{thm:tradeoff}] \label{thm:p-tradeoff}
Given a black-box prediction oracle, there exist
forecasting algorithms that satisfy the following
guarantees. For any $p \ge 2$, the algorithm
simultaneously achieves regret and $\ell_p$-calibration
error less than $O(T^{1/2})$.
For any $1 \leq p < 2$ and  
$x \in [ \frac{p}{p+2}, \frac{2p}{3p+2} ]$,
there exists a forecasting algorithm that simultaneously achieves regret $O(T^{-x})$ while keeping $\ell_p$-calibration error less than $O(T^{(2x-1)/(2/p \, - \, 1})$.
\end{theorem}
The proof of the theorem parallels the proof of
Theorem~\ref{thm:tradeoff}, substituting $m = \sqrt{T}$
in case $p \ge 2$, and $m = T^{(1-2x)/(2/p \, - \, 1)}$
in case $1 \le p < 2$.
}

%% file: sections/supplementary.tex
\section{Appendix}\label{sec:appendix}

% \subsection{Proof of $a=0$ case in Theorem \ref{thm:approach}} \label{pf:a=0}

% In the proof of Theorem~\ref{thm:approach} we consider a generic halfspace 
% $H \supseteq \approachset$, and we must prove that $H$ was approachable. 
% Denoting the inequality defining $H$ by $\langle a, x \rangle + bz \leq \theta$, 
% the case when $a \neq 0$ was covered in the proof presented in 
% Section~\ref{sec:approachability}, but we deferred the case $a=0$ to
% the appendix. In this case, the inequality defining $H$ is $bz \leq \theta$,
% or equivalently $z \leq \theta/b$. Since $H$ contains $\approachset$, 
% and $\approachset$ contains points whose $z$ coordinate is as large as
% $\frac{4 L_s}{m^2}$, we know that $\theta / b \geq \frac{4 L_s}{m^2}.$
% Therefore, to conclude the proof of the theorem, we must show that the
% forecaster can find a distribution $\wv \in \Delta_{m+1}$ such that 
% for each $y \in \{0,1\}$, $r_t(\wv,y) \le \frac{4 L_s}{m^2}.$ 
% In fact, as we show in the following lemma, the inequality can 
% be strengthened by a factor of two.

\begin{lemma} \label{lem:a=0}
    For every $q_t \in [0,1]$ there exists a 
    $\wv \in \Delta_{m+1}$ such that for 
    all $y \in \{0,1\}$, $r_t(\wv,y) \le \frac{2 L_s}{m^2}$. \footnote{$r_t(\wv,y)$ should be interpreted as $r_t(\wv,q_t, y)$}
\end{lemma}
\begin{proof}
    Fix $q = q_t$. Recalling the definition of 
    $r_t(\wv,y)$ in Equation~\eqref{eq:regret_game},
    we see that the lemma is equivalent to proving
    \begin{equation} \label{eq:a=0.1}
        \min_{\wv \in \Delta_{m+1}} 
        \max_{y \in \{0,1\}} 
        \sum_{i = 0}^m \wv_t(i) \left(S \left(\frac i m, y \right) - S(q, y)\right)
        \leq \frac{2 L_s}{m^2} .
    \end{equation}
    The functions $S \left( \frac i m, y \right)$ and 
    $S(q,y)$ appearing on the right side of~\eqref{eq:a=0.1}
    are affine functions of $y$, so we can enlarge the domain 
    of $y$ to be the compact, convex set $[0,1]$, rather than
    the two-element set $\{0,1\}$, and then apply von Neumann's
    Minimax Theorem to conclude that inequality~\eqref{eq:a=0.1}
    is equivalent to 
    \begin{equation} \label{eq:a=0.2}
        \max_{y \in [0,1]} 
        \min_{\wv \in \Delta_{m+1}} 
        \sum_{i = 0}^m \wv(i) \left(S \left(\frac i m, y \right) - S(q, y)\right)
        \leq \frac{2 L_s}{m^2} .
    \end{equation}
    The inequality~\eqref{eq:a=0.2} is easy to prove.
    For any $y \in [0,1]$, 
    choose $k \in [m]$ such that $|y - \frac k m| \le \frac1m$,
    and let $x = \frac k m$.
    By the Lipschitz property of $S$ we have
    \[
        | S(x,0) - S(y,0) | \le L_s |x-y| \le \frac{L_s}{m} \quad \mbox{and} \quad
        | S(y,1) - S(x,1) | \le L_s |x-y| \le \frac{L_s}{m},
    \]
    so the triangle inequality implies
    \[
        | S(x,0) - S(y,0) + S(y,1) - S(x,1) | \le \frac{2 L_s}{m},
    \]
    and hence
    \begin{equation} \label{eq:a=0.3}
        (x-y)[S(x,0) - S(y,0) + S(y,1) - S(x,1)] \le \frac{1}{m} \cdot \frac{2 L_s}{m} = \frac{2 L_s}{m^2} .
    \end{equation}
    Now, using the fact that $S$ is a strictly proper scoring rule 
    we have 
    \begin{align*} 
        S(x,y) - S(q,y) & \le
        S(x,y) - S(q,y) + [S(q,y) - S(y,y)] + [S(y,x) - S(x,x)] \\
        & = [S(x,y) - S(y,y)] + [S(y,x) - S(x,x)] \\
        & = (1-y) [S(x,0) - S(y,0)] + y [S(x,1) - S(y,1)] \\
        &    + (1-x) [S(y,0) - S(x,0)] + x [S(y,1) - S(x,1)] \\[1ex]
        & = (x-y) [S(x,0) - S(y,0) + S(y,1) - S(x,1)] \le \frac{2 L_s}{m^2} .
    \end{align*}
    Therefore, if we set $\wv$ to be the probability vector defined by 
    $\wv(k)=1$ and $\wv(j)=0$ for all $j \neq k$, we have
    \[
        \sum_{i = 0}^m \wv(i) \left(S \left(\frac i m, y \right) - S(q, y)\right)
        = \left(S \left(\frac k m, y \right) - S(q, y)\right)
        = S(x,y) - S(q,y) \le \frac{2 L_s}{m^2} .
    \]
    As $y \in [0,1]$ was arbitrary, we have shown that 
    inequality~\eqref{eq:a=0.2} holds, completing the proof
    of the lemma.
\end{proof}

\subsection{Proof of Lemma \ref{lem:smallest_halfspace}}\label{pf:smallest_halfspace}

\approachresult*

\begin{proof}
To show that $H_1$ is approachable, we will find a mixed distribution for the forecaster (i.e, a probability distribution over $p \in \{0/m, 1/m, 2/m, \ldots, (m-1)/m, 1\}$)	such that $\E_p \left[ \langle a, \lb_c(p, y) \rangle + b \lb_r(p, y) \right] \leq \frac{1}{m} + \frac{4 b L_S}{m^2}$ for any $y \in \{0,1\}$. 
For simplicity, define  
\begin{align*}
    c(i,y) = \frac{i}{m} \, - \, y 
    & \quad \text{and} \quad 
    C_i = \begin{bmatrix} c(i,0) \\ c(i,1) \end{bmatrix} \\
    d(i,y) = b \cdot S \left( \frac{i}{m}, y \right) - b \cdot S(q_t, y) 
    & \quad \text{and} \quad 
    D_i = \begin{bmatrix} d(i,0) \\ d(i,1) \end{bmatrix} \\
    f(i,y) = a_i c(i,y) + d(i,y) 
    & \quad \text{and} \quad 
    F_i = \begin{bmatrix} f(i,0) \\ f(i,1) \end{bmatrix}
    = a_i C_i + D_i
\end{align*}
Observe that $f(i,y) = \langle a, \lb_c(\frac{i}{m},y) \rangle
+ b \lb_r(\frac{i}{m},y)$, 
% d (i, y) = S \left(\frac{i}{m}, y \right) - S(q_t, y) \quad \text{and} \quad f (i, y) = a_i \left(\frac{i}{m} - y \right) + b d (i, y)
so our objective becomes to show that there exists a distribution $p$ over $\frac{i}{m} \in \{0, \ldots, m \}$ such that $\En_p f(i, y) \leq \frac{1}{m} + \frac{4 b L_S}{m^2}$ for $y \in \{0,1\}$, or equivalently that the vector $\En_p F_i$ belongs to the quadrant-shaped set $(-\infty,\frac1m + \frac{4 b L_S}{m^2}] \times (-\infty,\frac1m + \frac{4 b L_S}{m^2}]$.
% Since $y \in \{0,1\}$, it is easier to switch to the vector equivalent of these quantities. Therefore, let 
% $P_i = \begin{bmatrix} b d(i, 0) \\ b d(i, 1) \end{bmatrix}$, $L_i = \begin{bmatrix} \frac{i}{m} \\ \frac{i}{m} - 1 \end{bmatrix}$ and $O_i = a_i L_i$ so that $F_i = P_i + O_i = \begin{bmatrix} f(i, 0) \\ f(i, 1) \end{bmatrix}$ \\
% \[
% P_i = \begin{bmatrix} b d(i, 0) \\ b d(i, 1) \end{bmatrix}, \quad L_i = \begin{bmatrix} \frac{i}{m} \\ \frac{i}{m} - 1 \end{bmatrix} \quad \text{and} \quad O_i = a_i L_i \quad \text{so that} \quad F_i = P_i + O_i = \begin{bmatrix} f(i, 0) \\ f(i, 1) \end{bmatrix}
% \]
We will be choosing $p$ to be either a point-mass on 
$\frac{i}{m}$ for some $i$, or a distribution on two
consecutive values in the set $\{0,\frac1m, \frac2m, \ldots, 1\}.$
Hence, the vector $\En_p F_i$ will belong to one of $m$ closed line segments forming a polygonal path through the vectors $F_0,F_1,\ldots,F_m$. Observe that $F_0$ belongs to the
left half-plane, i.e.~$f(0,0) \le 0$, because
\begin{align*}
    f(0,0) & = a_0 c(0,0) + d(0,0) =
    b (S(0,0) -  S(q_t, 0)) \le 0, 
\end{align*}
where the last inequality holds because $b \ge 0$
and $S$ is a proper scoring rule. Similarly, $F_m$
belongs to the lower half-plane, i.e.~$f(m,1) \le 0$,
because
\begin{align*}
    f(m,1) & = a_m c(m,1) + d(m,1) = 
    b (S(1,1) - S(q_t,1)) \le 0 .
\end{align*}
If $F_0$ or, respectively, $F_m$ belongs to the third
quadrant --- that is, the set $(-\infty,0] \times (-\infty,0]$ ---
then we choose $p$ to be a point-mass on 0 or 1, respectively.
The remaining case is that $F_0$ and $F_m$ belong to the sets $(-\infty,0] \times (0,\infty)$ and $(0,\infty) \times (-\infty,0],$
respectively. In that case, $F_0$ and $F_m$ lie on opposite sides
of the line $L$ consisting of all points 
$\begin{bmatrix} x_0 \\ x_1 \end{bmatrix}$
that satisfy $x_0 = x_1$; $F_0$ lies above $L$ while $F_m$
lies below it. Hence, there must be at least one
index $j$ such that $F_j$ lies on or above $L$ while 
$F_{j+1}$ lies below it. We aim to construct a distribution $p$ 
supported on $\{ \frac{j}{m}, \, \frac{j+1}{m} \}$
such that $\En_p F_i$ belongs to the set $(-\infty,\frac1m + \frac{4 b L_S}{m^2}] \times (-\infty,\frac1m + \frac{4 b L_S}{m^2}]$.
Assume without loss of generality that $j \ge m/2$.
(The case $j \le m/2$ is handled symmetrically, by exchanging
the roles of the labels $y=0$ and $y=1$, i.e.~the first and 
second coordinates of the vectors we are considering.)

\begin{center}
\includegraphics[scale=0.4]{sections/illustration.png}
\end{center}
For $p \in [0,1]$ let $D(p)$ denote the vector
\[
    D(p) = b \cdot \begin{bmatrix}
        S(p,0) - S(q_t,0) \\ S(p,1) - S(q_t,1)
    \end{bmatrix}
\]
and observe that the notation $D_i$ defined earlier
is equivalent to $D(i/m).$
The fact that $S$ is a proper scoring rule ensures that 
when $y$ is a random sample from $\{0,1\}$ taking the 
value 1 with some probability $p$,
the value of $p'$ that minimizes $\En_y [S(p',y) - S(q_t,y)]$ 
is $p' = p$. Since the expected value 
$\En_y [S(p',y) - S(q_t,y)]$ is 
calculated by taking the inner product of 
the vector $D(p')$ with the probability vector
\[ 
  Y(p) = \begin{bmatrix} 1-p \\ p \end{bmatrix},
\]
this means that the curve 
$\mathcal{D} = \{ D(p') \mid 0 \le p' \le 1 \}$
is convex and that the line
$$L(p) = \{ x \mid \langle Y(p),x \rangle 
= \langle Y(p), D(p) \rangle \}$$ is 
tangent to $\mathcal{D}$
at the point $D(p)$. The normal
vector to this tangent line is 
$Y(p)$, so the vector 
$C(p) = \begin{bmatrix} p \\ p-1 \end{bmatrix}$,
being orthogonal to $Y(p)$, is parallel to the 
tangent line at $D(p).$ When $p=i/m$, observe 
that the vector $C(p)$ defined here coincides 
with $C_i$ defined earlier. 

Summarizing the foregoing discussion, 
the line 
$L(j/m) = \{ D_j + \lambda C_j \mid \lambda \in \mathbb{R} \}$
is tangent to the convex curve $\mathcal{D}$ at the point
$D_j$, hence it lies (weakly) below that curve. In particular,
recall the line $L$ 
consisting of points whose first and second coordinates
are equal, and consider the point $I_j$ where $L$ intersects
$L(j/m)$. Since $L(j/m)$ lies (weakly) below $\mathcal{D}$
and $\mathcal{D}$ intersects $L$ at $D(q_t) = \begin{bmatrix}
0 \\ 0 \end{bmatrix}$, the intersection of $L(j/m)$ with $L$ 
must belong to the third quadrant. 
%We will be defining a vector $G_j$ such that:
%\begin{itemize}
%    \item The line segment joining $F_j$ to $G_j$ 
%    is contained in $L(j/m)$ and contains $I_j$. 
%    \item $G_j$ is near $F_{j+1}$. 
%\end{itemize}
From these properties, it will follow that the
line segment joining $F_j$ to $F_{j+1}$ intersects
the set
$(-\infty,\frac1m + \frac{4 b L_S}{m^2}] \times (-\infty,\frac1m + \frac{4 b L_S}{m^2}]$ as required.

\begin{center}
\includegraphics[scale=0.6]{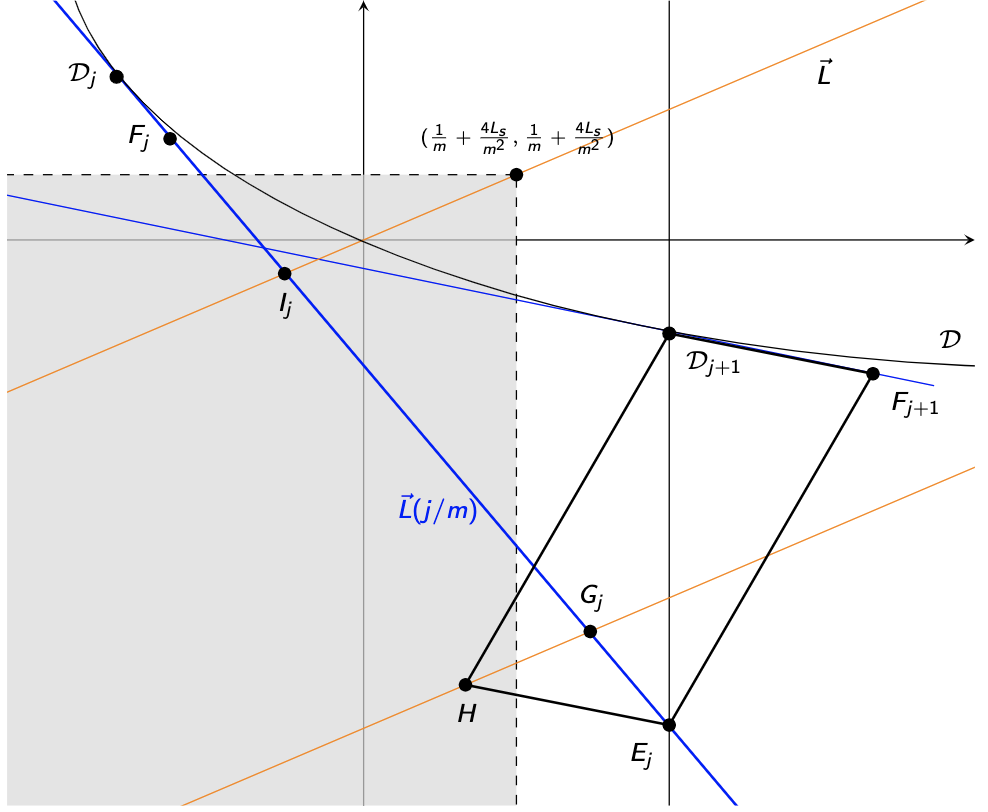}
\end{center}

Let $E_{j}$ be the intersection point of $L(j/m)$ with
a vertical line through $D_{j+1}.$ Since $L(j/m)$ lies
below $\mathcal{D}$, we know that $E_{j}$ is situated 
directly below $D_{j+1}.$ To reason about the distance between
$D_{j+1}$ and $E_j$, observe that the convexity of the 
curve $\mathcal{D}$ implies that the slope of the line segment
joining $D_j$ to $D_{j+1}$ lies between the slopes of the 
tangent lines at $D_j$ and $D_{j+1}$. Those slopes are 
$1 - m/j$ and $1 - m/(j+1),$ respectively. Hence, a pair
of lines passing through $D_j$, with slopes 
$1 - m/j$ and $1 - m/(j+1)$, will intersect 
the vertical line through $D_{j+1}$ in a line
segment that contains $D_{j+1}.$ The lower
endpoint of that line segment is $E_j$. Its
length is the
difference between the slopes of the two lines, 
times the horizontal displacement
between $D_j$ and $D_{j+1}$. In other
words, the length of the vertical line segment is
\[
    \frac{m}{j(j+1)} \cdot b \cdot \left[
    S \left( \frac{j}{m}, 0 \right) - 
    S \left( \frac{j+1}{m}, 0 \right) \right] 
    \; \le \; 
    \frac{m}{(m/2)^2} \cdot 
    \frac{b L_S}{m} =
    \frac{4 b L_S}{m^2} .
\]
Since the vertical line segment contains 
$E_j$ and $D_{j+1}$, its length is an upper bound
on their distance from one another.

Now define 
\[ G_j = E_j + a_{j+1} C_j = F_{j+1} + (E_j - D_{j+1}) + a_{j+1} 
    (C_j - C_{j+1}) .
\] 
Since $E_j$ lies on $L(j/m)$ and $C_j$ is parallel to $L(j/m)$,
we know that $G_j$ lies on $L(j/m).$ To determine the position
of $G_j$ relative to $L$, observe that $C_j - C_{j+1} = \begin{bmatrix} -1/m \\ -1/m \end{bmatrix}$
is parallel to $L$, $F_{j+1} + (E_j - D_{j+1})$ lies below $F_{j+1}$, and recall that $F_{j+1}$ lies below $L$.
Hence, $G_j$ 
lies below $L$. As $F_j$ lies on or above $L$ it follows
that the line segment joining $F_j$ to $G_j$ intersects
$L$, and this intersection point must be $I_j$ because
the segment connecting $F_j$ to $G_j$ is contained 
in $L(j/m)$.  Write $I_j = (1-t) F_j + t G_j$ for
some parameter $t \in [0,1].$

If $p$ is the distribution that selects a random
$\frac{i}{m} \in \{\frac{j}{m}, \frac{j+1}{m} \}$
by setting $\frac{i}{m} = \frac{j}{m}$ with probability 
$1-t$ and $\frac{i}{m} = \frac{j+1}{m}$ with probability
$t$, then 
\begin{align} 
\nonumber
    \En_p F_i = (1-t) F_j + t F_{j+1} & =
        (1-t) F_j + t G_j + t (F_{j+1} - G_j) \\
        \nonumber
        & = 
        I_j + t (D_{j+1} - E_j) + t a_{j+1} (C_{j+1} - C_j)\\
        \label{eq:enpfi}
        & =
        I_j + t (D_{j+1} - E_j) + \tfrac{t a_{j+1}}{m} 
        \begin{bmatrix} 1 \\ 1 \end{bmatrix} .
\end{align}
We need to show that both coordinates of the vector 
in Equation~\eqref{eq:enpfi} are less than or equal
to $\frac1m + \frac{4 b L_S}{m}.$ The first coordinate
of $I_j$ is non-positive, the first coordinate of $D_{j+1}-E_j$
is zero, and the first coordinate of $\tfrac{t a_{j+1}}{m}
\begin{bmatrix} 1 \\ 1 \end{bmatrix}$ is at most $\frac1m$
since $0 \le t \le 1$ and $|a_{j+1}| \le 1$. 
The second coordinate of $I_j$ is non-positive,
the second coordinate of $D_{j+1} - E_j$ is 
at most $\frac{4 b L_S}{m^2}$, and the second
coordinate of $\tfrac{t a_{j+1}}{m}
\begin{bmatrix} 1 \\ 1 \end{bmatrix}$ is at most $\frac1m$.
\end{proof}

\subsection{Constructing the Halfspace Oracle}\label{sec:oracle}
Here we go into more detail about how to construct the oracle asserted in \rdkedit{Section}~\ref{sec:algorithm}. Recall that in the proof of Lemma~\ref{pf:smallest_halfspace}, given a halfspace $\theta$ parameterized by $(a, b)$, we defined the vector $F_i$ as follows:
\begin{equation}
    F_i = \begin{bmatrix} f(i, 0) \\ f(i, 1) \end{bmatrix} \quad \text{where} \quad f (i, y) = a_i \left(\frac{i}{m} - y \right) + b \left[ S \left(\frac{i}{m}, y \right) - S(q_t, y) \right]
\end{equation}
In the proof, we note that $F_0$ is either in the 2nd or 3rd quadrant. Similarly, $F_m$ is either in the 3rd or 4th quadrant. Thus, we first check if $F_0$ or $F_m$ is in the 3rd quadrant. If one of them is, then we output a point distribution at the corresponding probability value.
If none of $F_0$ or $F_m$ is in the 3rd quadrant, then we binary search for an index $i$ with $F_i$ in the 3rd quadrant or a pair of consecutive indices $j, j+1$ where $F_j$ is in 2nd quadrant and $F_{j+1}$ is in the 4th quadrant. In the first case, $w_t(i) = 1$ and 0 everywhere else. In the second case, we set 
\begin{align}
    w_t(j) = \frac{f(j+1,1) - f(j+1, 0)}{f(j, 0) - f(j+1, 0) - f(j,1) + f(j+1,1)} \\
    w_t(j+1) = \frac{f(j,0) - f(j, 1)}{f(j, 0) - f(j+1, 0) - f(j,1) + f(j+1,1)}
\end{align}
and 0 everywhere else. 
The correctness of this procedure follows from the proof of Lemma~\ref{pf:smallest_halfspace}. The formula is obtained by solving this system of equations below to obtain a convex combination of $F_j$ and $F_{j+1}$:
\[
\begin{bmatrix}
  1 & 1
\end{bmatrix}
\begin{bmatrix} w_t(j) \\ w_t(j+1) \end{bmatrix} = 1
\quad \text{and} \quad
\begin{bmatrix}
  1 & -1
\end{bmatrix}
\begin{bmatrix}
  f(j, 0) & f(j+1, 0) \\
  f(j, 1) & f(j+1, 1) \\
\end{bmatrix}
\begin{bmatrix} w_t(j) \\ w_t(j+1) \end{bmatrix} = 0 
\]
Note that $F_i$ does not need to be pre-computed for every index. It can be computed online during the binary search steps. Thus, this halfspace oracle can be implemented in $O(\log m)$ steps. 

\pcodelete{
\subsection{Learning Algorithm for general $\ell_p$}\label{sec:lp_algorithm}
Here we describe the learning for general $\ell_p$. Similar to that of the learning algorithm for $p=2$, we use Online Gradient Descent algorithm from \citep{zinkevich_online_2003}.
\pcoedit{
As \citet{pmlr-v19-abernethy11b} points out in the calibration version, the vectors $\lb_t (\wv_t, y_t)$ are sparse and have at most two nonzero coordinates. This is also true for our recalibration payoff game since $\wv_t$ is a distribution over two consecutive probabilities values. Hence, the Gradient Descent Step requires only $O(1)$ computation and the Projection Step can be performed $O(1)$ (for $p = 1$). Since $\thb$ is the only state the OGD algorithm needs to store, the storage space required is $O(\min\{T, m\})$.}
The $\ell_2$-diameter for the convex set $\K_q^m$. The $\ell_2$ diameter of $\K_q^m$ for $q\leq 2$ (i.e $p \geq 2$) is upper bounded by $\sqrt{m}$. For $q > 2$ (i.e $p < 2$), the $\ell_2$ diameter $m^c$ where $c = \frac{1}{2} - \frac{1}{q}$. Since the vectors $\lb_t (\wv_t, y_t)$ are sparse and have at most two nonzero coordinates, a similar argument to the case where $p = 1$, allows us to show that the Gradient Descent Step and the Projection Step can be performed $O(1)$. 
\begin{align}
  C_{p,T}^{\epsilon, \delta}(\A) &\leq \dist_p^1 \left( \frac{1}{T} \sum_{t=1}^T \lb_t(\wv_t, y_t), \mathcal{S}_p^m \right) \\
  &\leq \frac{\text{Regret}_t}{T} \\
  &\leq \frac{GD}{\sqrt{T}} = O\left( \frac{m^c}{\sqrt{T}}\right) 
\end{align}
where $c = \frac{1}{2} - \frac{1}{q}$ for $q > 2$ (i.e $p < 2$) and $c = \frac{1}{2}$ for $q\leq 2$ (i.e $p \geq 2$)
}

\subsection{Proof of Lemma~\ref{lem:dualsetopt}}\label{sec:dualsetopt}
\dualsetopt*
\begin{proof}
\begin{align}
    \dist_1 \left( \xb, \approachset \right) &= \dist_1 \left( \xb_{1:m+1}, B_1^{m+1}(1/m) \right) + \dist_1 \left( \xb_{m+2}, \left( -\infty, \frac{4L_s}{m^2} \right] \right) \\
    &= -\frac{1}{m} - \min_{\theta: \norm{\theta}_{\infty} \leq 1} \langle - \xb_{1:m+1}, \theta \rangle  - \frac{4L_s}{m^2} - \min_{\theta \in [0,1]} \langle - \xb_{m+2}, \theta \rangle \\
    &= -\frac{1}{m} - \frac{4L_s}{m^2} - \min_{\theta \in \K} \langle - \xb, \theta \rangle \\
\end{align}
Remark: We need $\| \xb_{1:m+1}\|_{1} \geq 1 / m$, and $| \xb_{m+2} | \geq \frac{L_s}{m^2}$ mainly for technicality in order to ensure equality. If these didn't hold, just like in the proof of Approachability, if you're already in the set you wish to approach, you can just make an arbitrary move. Similarly, if $\| \xb_{1:m+1}\|_{1} < 1 / m$ (i.e calibration error is already less than $\frac{1}{m}$), the algorithm can just follow the oracle's predictions. On the other hand, if $\xb_{m+2} < \frac{L_s}{m^2}$, then following the halfspace oracle still ensures expected calibration error of at most $\frac{1}{m}$ for the timestep.
\end{proof}

% \begin{lemma}[Restating Lemma~\ref{lem:rateisapproach}]
% \begin{equation}
%   C_{p,T}^{\epsilon, \delta}(\A) = \dist_p^\infty \left( \vb_T, \mathcal{S}_p^{\lfloor \epsilon^{-1} \rfloor} (\epsilon, \delta) \right)
% \end{equation}
% \end{lemma}
% \begin{proof}
% \begin{align}
%     \dist_p^\infty \left( \vb_T, \mathcal{S}_p^{\lfloor \epsilon^{-1} \rfloor} (\epsilon, \delta) \right) &= \max \left\{ \dist_p \left( \mathbf{c}_T, B_p^{\lfloor \epsilon^{-1} \rfloor}(\epsilon/2) \right), \dist_1 \left( R_T, \left( -\infty, \frac{\delta}{2} \right] \right) \right\} \\
%     &= \max \left\{C_{p,T}^{\epsilon}(\A), R_T(\A) - \frac{\delta}{2}\right\} \\
%     &= C_{p,T}^{\epsilon, \delta}(\A) 
% \end{align}
% \end{proof}

%% file: sections/minimax-appendix.tex
\section{Reducing Recalibration to Online Multiobjective Optimization}
\label{sec:minimax-appendix}

In this section we present a reduction from recalibration to the online multiobjective optimization problem studied by \citet{multicalibeating}. We begin by reviewing their assumptions and terminology and restating their main result. 

\subsection{Review of Online Multiobjective Optimization}
\label{sec:minimax-appendix-review}
In the setting considered by \citet{multicalibeating}, a learner and an adversary play a $T$-round game where the timing and information structure of each round $t$ are as follows.
\begin{enumerate}
\item The adversary presents to the learner an {\em environment} $(X^t, Y^t, \ell^t)$ where each of $X^t, Y^t$ is a compact convex subset of a finite-dimensional Euclidean space, and $\ell^t$ is a continuous vector-valued loss function taking values in $[-C,C]^d$, such that each coordinate function $\ell^t_j(x,y) \, (j = 1,2,\ldots,d)$ is convex in its first argument and concave in its second argument. 
\item The learner chooses $x^t \in X^t$ and reveals it to the adversary.
\item The adversary chooses $y^t \in Y^t$ and reveals it to the learner.
\end{enumerate}
The game ends after $T$ rounds, and the cumulative loss vector is $\sum_{t=1}^T \ell^t(x^t,y^t)$. The learner's objective is to minimize the maximum coordinate of this vector.

The algorithm analyzed by \citet{multicalibeating} is easy to describe. For a specified learning rate $\eta > 0$, the algorithm computes in each round a weight vector $\chi^t$ whose $j^{\mathrm{th}}$ coordinate is proportional to $\exp(\eta \sum_{s=1}^{t-1} \ell_j^s(x^s,y^s)).$ Then it chooses $x^t$ by solving the minimax problem
\[
    x^t \in \arg \min_{x \in X^t} \max_{y \in Y^t}
    \langle \chi^t, \, \ell^t(x,y) \rangle .
\]
The analysis of the algorithm relates the learner's loss to a quantity called the {\em AMF value}, defined as follows. For the environment $(X^t, Y^t, \ell^t)$ selected by the adversary at time $t$, define $w_A^t$ by 
\[
    w_A^t = \sup_{y \in Y^t} \min_{x \in X^t} \left\{
    \max_{j \in [d]} \ell_j^t(x,y) \right\} .
\]
This quantity $w_A^t$ is called 
the {\em AMF value} of the stage-$t$ environment
$(X^t,Y^t,\ell^t)$ because it is the value of the game in which the adversary moves first, announcing $y \in Y^t$, the learner responds by selecting $x \in X^t$, and the learner seeks to minimize the loss function $\max_{j \in [d]} \ell_j^t(x,y).$
(The abbreviation ``AMF'' stands for ``adversary moves first''.)
% For a game transcript $\pi^T = (X^t,Y^t,\ell^t,x^t,y^t)_{t=1}^T$
% the AMF regret $R^T(\pi^T)$ is defined by
% \[
%     R^T(\pi^T)
% \]
\begin{theorem}[\citet{multicalibeating}]
\label{thm:lee}
    Suppose $T \geq \ln(d)$. 
    If the learner uses the multiplicative-weights 
    algorithm described above, with learning rate 
    $\eta = \sqrt{\frac{\ln d}{4 T C^2}}$, then its
    cumulative loss vector will satisfy
    \begin{equation} \label{eq:amf-regret}
        \max_{j \in [d]}
        \sum_{t=1}^T \ell_j^t(x^t,y^t) \le
        \sum_{t=1}^T w_A^t \; + \; 4 C 
        \sqrt{T \ln d} .
    \end{equation}
\end{theorem}

\subsection{Reducing Recalibration to Online Multiobjective Optimization}
\label{sec:minimax-appendix-reducing}

Suppose we are given $\varepsilon = 1/m$ for some 
natural number $m$, and we wish to design a recalibration
algorithm that predicts probabilities $p_t$ in the set 
$\{0, \, 1/m, \, 2/m, \, \ldots, 1 \}.$ Recall the 
vector-payoff game from Section~\ref{sec:approachability} 
that was used for recalibration. Adjusting notation 
to match the notation from \citet{multicalibeating},
the forecasting algorithm uses a distribution 
$x^t$ drawn from $X^t = \Delta_{m+1}$, the set of probability 
distributions on the $(m+1)$-element set 
$\{0, 1/m, \ldots, 1\}$. (In Section~\ref{sec:approachability}
this distribution was called $\mathbf{w}_t$.) The 
adversary selects $y^t$ from $Y^t = [0,1]$. (Formerly 
this was called $y_t$ and constrained to belong to 
$\{0,1\}$.) The vector payoff 
$\ell^t(x,y)$ is defined to be
$\ell^t(x,y) = \cbb (x,y) \oplus r_t(x,y)$,
where 
\begin{align}
    \cbb_i(x,y) &= x_i \left( y - \frac{i}{m} \right) 
    \label{eq:cbbj}
    \\
    r_t(x,y) &= \sum_{i=0}^m x_i \left(S \left(\frac i m, y \right) - S(q_t, y)\right) .
\end{align}
After $T$ rounds of interaction, if we write the 
average loss vector $\bar{\ell} = \frac1T \sum_{t=1}^T \ell^t(x^t,y^t)$ 
as 
\[
\bar{\ell} = 
\left( \frac1T \sum_{t=1}^T \cbb(x^t,y^t) \right)
\oplus \left( \frac1T \sum_{t=1}^T r_t(x^t,y^t) \right)
= \bar{\cbb} \oplus \bar{r}
\] 
then $\ell_1$ calibration error is $\| \bar{\cbb} \|_1$
while the average regret is $\bar{r}$.

The objective in the recalibration problem is to 
ensure that $\| \bar{\cbb} \|_1$ and $\bar{r}$ are
both small. This doesn't quite correspond to the
learner's goal in online multiobjective optimization,
which is to make every {\em coordinate} of the 
average loss vector small. The difference is that
in recalibration we are concerned with 
$\| \bar{\cbb} \|_1$ rather than $\| \bar{\cbb} \|_{\infty}$.
However, the difference can be overcome by embedding 
the loss vectors in a higher dimension. Specifically,
let $d = 2^{m+1} + 1$ and let $M$ be the matrix with
$d$ rows and $m+2$ columns such that the first 
$d-1$ rows of $M$ constitute the set of row
vectors $\{\pm 1\}^{m+1} \oplus (0)$ while the
last row of $M$ is the row vector 
$(0)^{m+1} \oplus (1) = (0,0,\ldots, 0,1).$
For any vector $w = c \oplus r \in \RR^{m+1} \oplus \RR$
we have
\begin{equation} \label{eq:norm-mw}
    \max_{j \in [d]} (Mw)_j =
    \max \{ \| c \|_1, \, r \} .
\end{equation}
Hence, in the online multiobjective optimization 
problem with $d$-dimensional 
vector losses $\tilde{\ell}^t = M \ell^t$,
the maximum coordinate of the (normalized)
cumulative loss vector $\frac1T \sum_{t=1}^T 
\tilde{\ell}^t(x^t,y^t)$ equals 
the maximum of the forecaster's 
$\ell_1$ calibration error and 
average regret.

To apply Theorem~\ref{thm:lee} to the
sequence of environments $(X^t, Y^t, \tilde{\ell}^t)$
we first need upper bounds on the infinity-norms 
of the loss vectors $\tilde{\ell}^t(x,y)$ and on
the AMF values,
$w_A^t$, of these environments. Such upper
bounds are very easy to obtain. 
We have 
\[
    \| \tilde{\ell}^t(x,y) \|_{\infty} = 
    \max \{ \|\cbb(x,y)\|_1, \, |r_t(x,t)| \} 
    \le \max \{ 1, L_s \},
\]
where the inequality follows from
the definitions of $\cbb(x,y)$ and $r_t(x,y)$,
recalling that the Lipschitz constant of the 
scoring rule $S$ is $L_s$. As for bounding the
AMF values, for each $y \in [0,1]$,
if we let $\frac{i}{m}$ be the element of 
$\{0, \, 1/m, \, 2/m, \, \ldots, \, 1\}$
closest to $y$, then $|y - \frac{i}{m}| \le \frac{1}{2m}.$
Define $x \in \Delta_{m+1}$ to be a point-mass
distribution on $i$. Then, 
\[
  \| \cbb(x,y) \|_1 = \left| y - \frac{i}{m} \right| 
  \le \frac{1}{2m},
\]
Meanwhile, 
\[
  r_t(x,y) = S \left( \frac i m , y \right) -
  S \left( q_t, y \right) 
  \le 
  S \left( \frac i m, y \right) - S \left( y, y \right)  
  \le
  L_s \left| \frac i m - y \right| 
  \le
  \frac{L_s}{2m} .
\]
Hence,
\[
    w_A^t = \sup_{y \in [0,1]} \min_{x \in \Delta_{m+1}} 
    \left\{ \max_{j \in [d]} \tilde{\ell}^t_j(x,y) 
    \right\} =
    \sup_{y \in [0,1]} \min_{x \in \Delta_{m+1}} 
    \left\{ \max \left( 
        \| \cbb(x,y) \|_1 , \, r_t(x,t) 
    \right) \right\}
    \le
    \frac{\max(1,L_s)}{2m}.
\]

Using the upper bounds
$C \le \max (1,L_s)$ and $w_A^t \le 
\frac{\max(1,L_s)}{2m}$ in
Theorem~\ref{thm:lee}, we find that if the
algorithm of \citet{multicalibeating} with
learning rate $\eta = \sqrt{\frac{\ln d}{4C T^2}}$
is applied to the sequence of environments 
$(X^t, Y^t, \tilde{\ell}^t)$ it will satisfy
the bound
\[
    \max_{j \in [d]} 
    \left( \frac1T \sum_{t=1}^T \tilde{\ell^t}_j(x^t,y^t)
    \right) 
    \le 
    \frac1T \sum_{t=1}^T w_A^t \; + \; 
    \frac1T \cdot 4C \sqrt{T \ln d} 
    \le 
    \max(1,L_s) \cdot 
    \left( \frac{1}{2m} \, + \, 
    4 \sqrt{ \frac{\ln d}{T} } \right) .
\]
Earlier we derived that the 
left side is the maximum of the 
forecaster's $\ell_1$ calibration error
and average regret. Recalling that 
$1/m = \varepsilon$ and that 
$d = 2^{m+1} + 1,$ we find that 
both the $\ell_1$ calibration error
and the average regret are bounded 
above by $\max(1,L_s) \cdot O(\varepsilon + 1 / 
\sqrt{\varepsilon T}))$. 

Compared to this bound, our Algorithm~\ref{alg:recalib} 
achieves the same upper bound on $\ell_1$ calibration
error but an improved bound of 
$O(L_s \varepsilon^2 + 1/\sqrt{\varepsilon T})$
on average regret. It is tempting to try 
to modify the reduction from recalibration to 
online multiobjective optimization, to see 
if it can achieve the same bound. For example, 
above when we derived the inequality 
$r_t(x,y) \le \frac{L_s}{2m},$ a more 
refined analysis using the property that 
the scoring rule $S$ is strictly proper 
would yield the bound $r_t(x,y) \le O \left( 
\frac{L_s}{m^2} \right).$
This means one could modify the definition of the loss
vectors $\tilde{\ell}^t(x,y)$ by rescaling their
final coordinate to equal $m \cdot r_t(x,y)$
rather than $r_t(x,y),$ without invalidating
the upper bound on the AMF values $w_A^t.$
Then an upper bound of the 
form $\max_j \left( \frac1T \sum_t \tilde{\ell}_j^t(x^t,y^t)
\right) \le O(\varepsilon)$  
would simultaneously imply 
$\ell_1$ calibration error $O(\varepsilon)$
and average regret $O(\varepsilon^2),$ because
one gains a factor of $1/m = \varepsilon$ when
rescaling the final coordinate of 
$\frac1T \sum_t \tilde{\ell}_j^t(x^t,y^t)$
to convert it back into average regret.
However, defining the final coordinate of 
$\tilde{\ell}^t(x,y)$ to equal 
$m \cdot r_t(x,y)$ would mean
that the infinity-norm of the loss vectors
is bounded above by $C = m \cdot \max(1,L_s)$,
it is no longer bounded above merely by 
$\max(1,L_s).$ Hence, the rescaling inflates 
the regret term $4 C \sqrt{T \ln d}$ in 
Theorem~\ref{thm:lee} by a factor of $m = 1/\varepsilon$,
more than offsetting any potential gains 
resulting from the rescaling.

\subsection{Efficient Implementation of the Reduction}
\label{sec:minimax-appendix-implementation}

Because the reduction described in Section~\ref{sec:minimax-appendix-reducing} involves loss vectors in dimension $d = 2^{m+1}+1$, a straightforward implementation of the 
reduction runs the risk of requiring running time $O(2^m)$ 
per iteration. Fortunately, there is an implementation requiring only $\mathrm{poly}(m)$ running time per iteration. The key to avoiding the exponential dependence on $m$ is, first of all, to store the vectors $\ell^t(x^t,y^t)$, which are only $(m+2)$-dimensional, rather than the exponentially higher-dimensional loss vectors $\tilde{\ell}(x^t,y^t).$ However, the algorithm still needs to compute 
\[
    x^t \in \arg \min_{x \in X^t} \max_{y \in Y^t}
    \langle \chi^t, \, \tilde{\ell}^t(x,y) \rangle 
\]
where $\chi^t$ is a $d$-dimensional vector with coordinates
\[
\chi^t_j = \frac{\exp(\eta \sum_{s=1}^{t-1} \tilde{\ell}_j^s(x^s,y^s))}{\sum_{i \in [d]} \exp(\eta \sum_{s=1}^{t-1} \tilde{\ell}_i^s(x^s,y^s))}.
\]
Expanding out the inner product $\langle \chi^t, \, \tilde{\ell}^t(x,y) \rangle$ in the definition of $x^t$, we find that 
\begin{equation} \label{eq:bigsums}
    x^t \in \arg \min_{x \in X^t} \max_{y \in Y^t}
    \frac{\sum_{j \in [d]} \exp(\eta \sum_{s=1}^{t-1} \tilde{\ell}_j^s(x^s,y^s)) \cdot \tilde{\ell}_j^t(x^t,y^t)}{\sum_{i \in [d]} \exp(\eta \sum_{s=1}^{t-1} \tilde{\ell}_i^s(x^s,y^s))}.
\end{equation}
To compute the sums in the numerator and denominator, recall that for each $j \in [d-1]$ there is a corresponding sign vector $\sigma \in \{\pm 1\}^{m+1}$ such that 
$\tilde{\ell}_j^s(x^s,y^s) = \sum_{k=1}^{m+1} \sigma_k \ell_k^s(x^s,y^s).$ Hence, the sum in the denominator 
of Equation~\eqref{eq:bigsums} simplifies as 
\begin{align} 
\nonumber
    \sum_{i \in [d]} \exp(\eta \sum_{s=1}^{t-1} \tilde{\ell}_i^s(x^s,y^s)) & = 
    \exp(\eta \sum_{s=1}^{t-1} \tilde{\ell}_d^s(x^s,y^s)) 
    + \sum_{\sigma \in \{\pm 1\}^{m+1}}
    \exp(\eta \sum_{s=1}^{t-1} \sum_{k=1}^{m+1} 
    \sigma_k \ell_k^s(x^s,y^s)) \\
\nonumber
    & = 
    \exp(\eta \sum_{s=1}^{t-1} \tilde{\ell}_d^s(x^s,y^s)) 
    + \sum_{\sigma \in \{\pm 1\}^{m+1}}
    \prod_{k=1}^{m+1}
    \exp(\eta \sigma_k \sum_{s=1}^{t-1} \ell_k^s(x^s,y^s)) \\
    & = 
    \exp(\eta \sum_{s=1}^{t-1} \tilde{\ell}_d^s(x^s,y^s)) 
    + \prod_{k=1}^{m+1} \sum_{\sigma_k \in \pm 1}
    \exp(\eta \sigma_k \sum_{s=1}^{t-1} \ell_k^s(x^s,y^s)) 
\label{eq:denominator}
\end{align}
and the formula on the last line can be computed in $O(mt)$
arithmetic operations. In fact, a further running time improvement can be achieved by using dynamic programming to amortize over time steps. At the end of each time step, if we store the quantities $\exp(\eta \sum_{s=1}^{t-1} \tilde{\ell}_d^s(x^s,y^s))$ and 
$\exp(\eta \sigma_k \sum_{s=1}^{t-1} \ell_k^s(x^s,y^s))$
for each $k \in [m+1]$ and $\sigma_k \in \pm 1$, then 
updating these values to incorporate the loss vector from time $s = t$ requires constant time (a single multiplicative update) for each of the $2m+3$ stored values. Evaluating the formula on the last line of Equation~\eqref{eq:denominator} then requires applying only $O(m)$ arithmetic operations to the stored values.

A similar simplification pertains to the numerator in 
Equation~\eqref{eq:bigsums}. To save space, we will ignore the $j=d$ term of the sum, which is a special case that can be computed separately from the terms corresponding to $j \in [d-1]$. As before, each $j \in [d-1]$ corresponds to a sign vector $\sigma \in \{ \pm 1 \}^{m+1}$.
\begin{align}
\nonumber
    \lefteqn{\sum_{\sigma \in \{ \pm 1 \}^{m+1}} 
    \exp \left( \eta \sum_{s=1}^{t-1} \sum_{k=1}^{m+1} 
    \sigma_k \ell_k^s(x^s,y^s) \right)
    \left( \sum_{k'=1}^{m+1} \sigma_{k'} \ell_{k'}^t(x^t,y^t)
    \right)} \\
\nonumber
    & = \sum_{k'=1}^{m+1}
    \sum_{\sigma \in \{ \pm 1 \}^{m+1}}
    \prod_{k=1}^{m+1}
    \exp \left( \eta \sum_{s=1}^{t-1} 
    \sigma_k \ell_k^s(x^s,y^s) \right)
    \sigma_{k'} \ell_{k'}^t(x^t,y^t) \\
\label{eq:numerator}
    & = \sum_{k'=1}^{m+1}
    \left( \sum_{\sigma_{k'} \in \{ \pm 1 \}}
        \exp \left( \eta \sigma_{k'} \sum_{s=1}^{t-1} 
        \ell_{k'}^s(x^s,y^s) \right) \sigma_{k'}
        \ell_{k'}^t(x^t,y^t) \right)
    \cdot
    \prod_{k \neq k'} 
    \left( \sum_{\sigma_{k} \in \{ \pm 1 \}}
        \exp \left( \eta \sigma_{k} \sum_{s=1}^{t-1} 
        \ell_{k}^s(x^s,y^s) \right)
    \right) .
\end{align}
As before, using dynamic programming the
expression on the last line can be 
computed using only $O(m)$ arithmetic operations per 
time step.